\documentclass[11pt,english]{article}

\usepackage[T1]{fontenc}
\usepackage[latin9,utf8]{inputenc}
\usepackage{geometry}
\usepackage{color}
\usepackage{babel}
\usepackage{array}
\usepackage{float}
\usepackage{calc}
\usepackage{mathtools}
\usepackage{enumitem}
\usepackage{multirow}
\usepackage{amsmath}
\usepackage{amsthm}
\usepackage{amssymb}
\usepackage[square,numbers,compress,sort]{natbib}
\usepackage{xargs}[2008/03/08]
\usepackage[unicode=true,pdfusetitle,
bookmarks=true,bookmarksnumbered=false,bookmarksopen=false,
breaklinks=false,pdfborder={0 0 1},backref=false,colorlinks=true]
{hyperref}
\hypersetup{
	citecolor=blue, linkcolor=red}

\makeatletter

\floatstyle{ruled}
\newfloat{algorithm}{tbp}{loa}
\providecommand{\algorithmname}{Algorithm}
\floatname{algorithm}{\protect\algorithmname}

\theoremstyle{plain}
\newtheorem{thm}{\protect\theoremname}
\theoremstyle{remark}

\theoremstyle{plain}

\theoremstyle{plain}

\theoremstyle{plain}
\newtheorem{lem}{Lemma}
\theoremstyle{plain}
\newtheorem{fact}{Fact}

\usepackage[mathscr]{euscript}
\usepackage{amsfonts}
\usepackage{dsfont}
\usepackage{times}
\usepackage{amsmath}
\usepackage{bbm}

\pdfmapfile{+txfonts.map}
\usepackage{txfonts}
\usepackage{appendix}
\allowdisplaybreaks

\renewcommand{\citep}{\cite}
\renewcommand{\citet}{\cite}

\usepackage{caption}
\usepackage{subcaption}

\usepackage{algorithm}
\usepackage{algorithmicx}\usepackage{algpseudocode}\usepackage{cleveref}

\newcommand{\algmargin}{\the\ALG@thistlm}

\newlength{\whilewidth}
\algdef{SE}[parWHILE]{parWhile}{EndparWhile}[1]
{\parbox[t]{\dimexpr\linewidth-\algmargin}{%
		\hangindent\whilewidth\strut\algorithmicwhile\ #1\ \algorithmicdo\strut}}{\algorithmicend\ \algorithmicwhile}%
\algnewcommand{\parState}[1]{\State%
	\parbox[t]{\dimexpr\linewidth-\algmargin}{\strut #1\strut}}

\providecommand{\theoremname}{Theorem}

\begin{document}
\global\long\def\cC{{\cal C}}%
\global\long\def\cM{{\cal M}}%
\global\long\def\cN{{\cal N}}%
\global\long\def\cV{{\cal V}}%
\global\long\def\cF{{\cal F}}%
\global\long\def\cR{{\cal R}}%
\global\long\def\cP{{\cal P}}%
\global\long\def\cG{{\cal G}}%
\global\long\def\cB{{\cal B}}%
\global\long\def\cD{{\cal D}}%
\global\long\def\cX{{\cal X}}%
\global\long\def\cS{{\cal S}}%
\global\long\def\cA{{\cal A}}%
\global\long\def\cY{{\cal Y}}%
\global\long\def\cL{{\cal L}}%
\global\long\def\cQ{{\cal Q}}%
\global\long\def\cU{{\cal U}}%
\global\long\def\real{\mathbb{R}}%
\global\long\def\E{\mathbb{E}}%
\global\long\def\T{\mathbb{T}}%
\global\long\def\P{\mathbb{P}}%
\global\long\def\pol{\pi}%
\global\long\def\indic{\mathbb{I}}%
\global\long\def\Ps{P^{\star}}%
\global\long\def\Php{\hat{P}^{\pol}}%
\global\long\def\Pp{P^{\pol}}%
\global\long\def\Vs{V^{\star}}%
\global\long\def\Vp{V^{\pol}}%
\global\long\def\Qs{Q^{\star}}%
\global\long\def\Qp{Q^{\pol}}%
\global\long\def\vhat{\hat{v}}%
\global\long\def\zhat{\hat{z}}%
\global\long\def\e{\mathbf{e}}%
\global\long\def\g{\mathbf{g}}%
\global\long\def\w{w}%
\global\long\def\v{\mathbf{v}}%
\global\long\def\fmap{\phi}%
\global\long\def\pmap{\mu}%
\global\long\def\IdMat{I}%
\global\long\def\bP{\mathbf{P}}%
\global\long\def\A{\mathbf{A}}%
\global\long\def\ucbrate{\lambda_{\text{UCB}}}%
\global\long\def\tmp{\beta}%
\global\long\def\diag{\mathop{\text{diag}}}%
\global\long\def\argmin{\mathop{\text{argmin}}}%
\global\long\def\argmax{\mathop{\text{argmax}}}%
\newcommandx\norm[2][usedefault, addprefix=\global, 1=\#1]{\ensuremath{\left\Vert #1\right\Vert {}_{#2}}}%
\global\long\def\conf{\text{conf}}%
\global\long\def\lse{\mathsf{lse}}%
\renewcommandx\norm[2][usedefault, addprefix=\global, 1=\#1]{\Vert#1\|_{#2}}%
\global\long\def\reg{\textup{Regret}}%
\global\long\def\Var{\textup{Var}}%
\global\long\def\sens{\text{sensitivity}}%
\global\long\def\edim{d_{E}}%
\global\long\def\rsvil{\mathsf{RSVI.L}}%
\global\long\def\rsvig{\mathsf{RSVI.G}}%
\global\long\def\mvi{\mathsf{MetaRSVI}}%
\global\long\def\kl{\mathsf{KL}}%
\global\long\def\condrew{\text{if }\beta>0}%
\global\long\def\condcost{\text{if }\beta<0}%

\title{Exponential Bellman Equation and Improved Regret Bounds for Risk-Sensitive Reinforcement Learning}

\author{Yingjie Fei$^1$,$\ $ Zhuoran Yang$^2$,$\ $ Yudong Chen$^3$,$\ $ and Zhaoran Wang$^1$\\
	\normalsize 
	$^1$ Northwestern University, 
	\normalsize 
	$^2$ Princeton University, 
	\normalsize 
	$^3$ University of Wisconsin-Madison\\
	\normalsize 
	yf275@cornell.edu, zy6@princeton.edu, yudong.chen@wisc.edu, zhaoranwang@gmail.com}

\date{}
\maketitle

\begin{abstract}

We study risk-sensitive reinforcement learning (RL) based on the entropic risk measure. Although existing works have established non-asymptotic regret guarantees for this problem, they leave open an exponential gap between the upper and lower bounds. We identify the deficiencies in existing algorithms and their analysis that result in such a gap. To remedy these deficiencies, we investigate a simple transformation of the risk-sensitive Bellman equations, which we call the exponential Bellman equation. The exponential Bellman equation inspires us to develop a novel analysis of Bellman backup procedures in risk-sensitive RL algorithms, and further motivates the design of a novel exploration mechanism. We show that these analytic and algorithmic innovations together lead to improved regret upper bounds over existing ones.
\end{abstract}

\section{Introduction \label{sec:intro}}
Risk-sensitive reinforcement learning (RL) is important for practical
and high-stake applications, such as  self-driving and robotic surgery. In contrast with standard and risk-neutral RL, it optimizes some risk measure
of cumulative rewards instead of their expectation. One foundational framework for risk-sensitive RL maximizes 
the entropic risk measure of the reward, which takes the form of 
\[
V^{\pi}=\frac{1}{\beta} \log\{\E_{\pi}[e^{\beta R}]\},
\]
with respect to the policy $\pi$, where $\beta\ne0$ is a given risk
parameter and $R$ denotes the cumulative rewards.

Recently, the works of \citet{fei2020risk,fei2021risk} investigate the online
setting of the above risk-sensitive RL problem. 
Under $K$-episode MDPs with horizon
length of $H$, they propose two model-free algorithms, namely RSVI and RSQ,
and prove that their algorithms achieve the regret upper bound (with its informal form given by)
\begin{align*}
\reg(K)\lesssim e^{|\beta| H^{2}} \cdot \frac{e^{|\beta| H}-1}{|\beta| H}\sqrt{\text{poly}(H)\cdot K}
\end{align*}
without assuming knowledge of the transition distribution or access
to a simulator.  They also provide a lower bound (informally presented as)
\begin{align*}
\reg(K)\gtrsim\frac{e^{|\beta| H'}-1}{|\beta| H}\sqrt{\text{poly}(H)\cdot K}
\end{align*}
that any algorithm has to incur, where $H'$ is a linear function
of $H$.  
Despite the non-asymptotic nature
of their results, it is not hard to see that a wide gap exists between
the two bounds. Specifically, the upper bound has an additional $e^{|\beta| H^{2}}$
factor compared to the lower bound, and even worse, this factor is
dominating in the upper bound since the quadratic exponent in $e^{|\beta| H^{2}}$
makes it exponentially larger than $\frac{e^{|\beta| H}-1}{|\beta| H}$
even for moderate values of $|\beta|$ and $H$. 
It is unclear whether the factor of $e^{|\beta|H^{2}}$ is intrinsic in the upper bound.

In this paper, we show that the additional factor in the upper bound is not intrinsic for the upper bound and can be eliminated by a refined algorithmic design and analysis. We identify two deficiencies in the existing algorithms and their analysis: (1) the main element of the analysis follows existing analysis of risk-neutral RL algorithms, which fails to exploit the special structure of the Bellman equations of risk-sensitive RL; (2) the existing algorithms use an excessively large bonus that results in the exponential blow-up in the regret upper bound.

To address the above shortcomings, we consider a simple transformation of the Bellman equations analyzed so far in the literature, which we call the \textit{exponential Bellman equation}. A distinctive feature of the exponential Bellman equation is that they associate the instantaneous reward and value function of the next step in a multiplicative way, rather than in an additive way as in the standard Bellman equations.
From the exponential Bellman equation, we develop a novel analysis
of the Bellman backup procedure for risk-sensitive RL algorithms that
are based on the principle of optimism. The analysis further motivates a novel exploration
mechanism called \emph{doubly decaying} bonus, which helps the algorithms
adapt to their estimation error over each horizon step while at the
same time exploring efficiently. These discoveries enable us to propose
two model-free algorithms for RL with the entropic risk measure based
on the novel bonus. By combining the new analysis and bonus design,
we prove that the preceding algorithms attain nearly optimal regret
bounds under episodic and finite-horizon MDPs. Compared to prior results,
our regret bounds feature an exponential improvement with respect to the horizon
length and risk parameter, removing the factor of $e^{|\beta| H^2}$ from existing upper bounds. This significantly narrows the gap between
upper bounds and the existing lower bound of regret.

In summary, we make the following theoretical contributions in this
paper.
\begin{enumerate}
	\item We investigate the gap between existing upper and lower regret bounds in the context of risk-sensitive RL, and identify deficiencies of the existing algorithms and analysis;
	\item We consider the exponential Bellman equation, which inspires us to propose a novel analysis
	of the Bellman backup procedure for RL algorithms based on the entropic
	risk measure. It further motivates a novel bonus design called
	doubly decaying bonus. We then design two model-free risk-sensitive RL
	algorithms equipped with the novel bonus. 
	\item The novel analytic framework and bonus design together enable us to
	prove that the preceding algorithms achieve nearly optimal regret
	bounds, which improve upon existing ones by an exponential factor in terms of the horizon
	length and risk sensitivity.
\end{enumerate}

\section{Related works \label{sec:related}}

The problem of RL with respect to the entropic risk measure is first proposed by the classical work of \citet{howard1972risk},
and has since inspired a large body of studies \citep{bauerle2014more,borkar2001sensitivity,borkar2002q,borkar2010learning,borkar2002risk,cavazos2011discounted,coraluppi1999chapter,di1999risk,di2000infinite,di2007infinite,fleming1995risk,hernandez1996risk,huang2020stochastic,jaskiewicz2007average,marcus1997risk,mihatsch2002risk,osogami2012robustness,patek2001terminating,shen2013risk,shen2014risk,whittle1990risk}.
However, the algorithms from this line of works require knowledge
of the transition kernel or assume access to a simulator of the underlying
environment. Theoretical properties of these algorithms are investigated
based on these assumptions, but the results are mostly of asymptotic
nature, which do not shed light on their dependency on key parameters
of the environment and agent. 

The work of \citet{fei2020risk} represents the first effort to investigate the setting where transitions are unknown and simulators of the environment are unavailable. It establishes the first non-asymptotic regret or sample complexity guarantees under the tabular setting. Building upon \citet{fei2020risk}, the authors of \citet{fei2021risk} extend the results to the function approximation setting, by considering linear and general function approximations of the underlying MDPs. Nevertheless, as discussed in Section \ref{sec:intro}, both works leave open an exponential gap between the regret upper and lower bounds, which the present work aims to address via novel algorithms and analysis motivated by the exponential Bellman equation.

We remark that although the exponential Bellman
equation has been previously investigated in the literature of risk-sensitive
RL \citep{borkar2002q,bauerle2014more}, this is the first
time that it is explored for deriving regret and sample complexity guarantees of risk-sensitive
RL algorithms. In Appendix \ref{sec:dist_RL}, we also make connections between risk-sensitive
RL and distributional RL through the exponential Bellman equation.

\paragraph{Notations.}

For a positive integer $n$, we let $[n]\coloneqq\{1,2,\ldots,n\}$.
For two non-negative sequences $\{a_{i}\}$ and $\{b_{i}\}$, we write
$a_{i}\lesssim b_{i}$ if there exists a universal constant $C>0$
such that $a_{i}\le Cb_{i}$ for all $i$, and write $a_{i}\asymp b_{i}$
if $a_{i}\lesssim b_{i}$ and $b_{i}\lesssim a_{i}$. We use $\tilde{O}(\cdot)$
to denote $O(\cdot)$ while hiding logarithmic factors. For functions
$f,g:\cU\to\real$, where $\cU$ denotes their domain, we write $f\ge g$
if $f(u)\ge g(u)$ for any $u\in\cU$. We denote by $\indic\{\cdot\}$
the indicator function.

\section{Problem background}

\subsection{Episodic and finite-horizon MDP }

The setting of episodic Markov decision processes can be denoted by
$\text{MDP}(\cS,\cA,H,\cP,\cR)$, where $\cS$ is the set of states,
$\cA$ is the set of actions, $H\in\mathbb{Z}_{>0}$ is the length
of each episode, and $\cP=\{P_{h}\}_{h\in[H]}$ and $\cR=\{r_{h}\}_{h\in[H]}$
are the sets of transition kernels and reward functions, respectively.
We let $S\coloneqq\left|\cS\right|$ and $A\coloneqq\left|\cA\right|$,
and we assume $S,A<\infty$. We let $P_{h}(\cdot\,|\,s,a)$
denote the probability distribution over successor states of step
$h+1$ if action $a$ is executed in state $s$ at step $h$. We assume that the reward function
$r_{h}:\cS\times\cA\to[0,1]$ is deterministic. We also assume that both $\cP$ and $\cR$ are unknown to learning agents.

Under the setting of an episodic MDP, the agent aims to learn the
optimal policy by interacting with the environment throughout $K>0$ episodes, described as follows. At the beginning of episode $k$,
an initial state $s_{1}^{k}$ is selected by the environment and we assume $s_{1}^k$ stays the same for all $k \in [K]$. In each
step $h\in[H]$ of episode $k$, the agent observes state $s_{h}^{k}\in\cS$,
executes an action $a_{h}^{k}\in\cA$, and receives a reward equal
to $r_{h}(s_{h}^{k},a_{h}^{k})$ from the environment. The MDP then
transitions into state $s_{h+1}^{k}$ randomly drawn from the transition
kernel $P_{h}(\cdot\,|\,s_{h}^{k},a_{h}^{k})$. The episode terminates
at step $H+1$, in which the agent does not take actions or receive
rewards. We define a policy $\pi=\{\pi_{h}\}_{h\in[H]}$ as a collection
of functions $\pi_{h}:\cS\to\cA$, where $\pi_{h}(s)$ is the action
that the agent takes in state $s$ at step $h$ of the episode.

\subsection{Risk-sensitive RL}

For each $h\in[H]$, we define the value function $V_{h}^{\pi}:\cS\to\real$
of a policy $\pi$ as the cumulative utility of the agent at state
$s$ of step $h$ under the entropic risk measure, assuming that the
agent commits to policy $\pi$ in later steps. Specifically, we define
\begin{equation}
\forall(h,s)\in[H]\times\cS,\quad V_{h}^{\pi}(s)\coloneqq\frac{1}{\tmp}\log\left\{ \E\left[e^{\tmp\sum_{i=h}^{H}r_{i}(s_{i},\pi_{i}(s_{i}))}\ \Big|\ s_{h}=s\right]\right\},\label{eq:value_func}
\end{equation}
where $\beta\ne0$ is a given risk parameter. The agent aims to maximize
his cumulative utility in step 1, that is, to find a policy $\pi$
such that $V_{1}^{\pi}(s)$ is maximized for all state $s\in\cS$.
Under this setting, if $\beta>0$, the agent is risk-seeking and if
$\beta<0$, the agent is risk-averse. Furthermore, as $\beta\to0$
the agent tends to be risk-neutral and $V_{h}^{\pi}(s)$ tends to
the classical value function.

We may also define the action-value function $Q_{h}^{\pi}:\cS\times\cA\to\real$,
which is the cumulative utility of the agent who follows policy $\pi$,
conditional on a particular state-action pair; formally, this is given
by 
\begin{equation}
\forall(h,s,a)\in[H]\times\cS\times\cA,\quad Q_{h}^{\pi}(s,a)\coloneqq\frac{1}{\tmp}\log\left\{ \E\left[e^{\tmp\sum_{i=h}^{H}r_{i}(s_{i},a_{i})}\ \Big|\ s_{h}=s,a_{h}=a\right]\right\}, \label{eq:action_value_func}
\end{equation}
Under some mild regularity conditions \citep{bauerle2014more}, there always exists an optimal 
policy, which we denote as $\pi^{*}$, that yields the optimal value
$V_{h}^{*}(s)\coloneqq\sup_{\pi}V_{h}^{\pi}(s)$ for all $(h,s)\in[H]\times\cS$. 

\paragraph{Bellman equations.}
For all $(s,a)\in\cS\times\cA$, the Bellman
equation associated with a policy $\pi$ is given by
\begin{align}
Q_{h}^{\pi}(s,a) & =r_{h}(s,a)+\frac{1}{\tmp}\log\left\{ \E_{s'\sim P_{h}(\cdot\,|\,s,a)}\left[e^{\beta\cdot V_{h+1}^{\pi}(s')}\right]\right\} ,\label{eq:bellman}\\
V_{h}^{\pi}(s) & =Q_{h}^{\pi}(s,\pol(s)),\qquad V_{H+1}^{\pi}(s)=0\nonumber 
\end{align}
for $h\in[H]$. In Equation \eqref{eq:bellman}, it can be seen that
the action value $Q_{h}^{\pi}$ of step $h$ is a non-linear function
of the value function $V_{h+1}^{\pi}$ of the later step. This is
in contrast with the linear Bellman equations in the risk-neutral
setting ($\beta \to 0$), where $Q_{h}^{\pi}(s,a)=r_{h}(s,a)+\E_{s'}[V_{h+1}^{\pi}(s')]$.
Based on Equation \eqref{eq:bellman}, for $h\in[H]$, the Bellman
optimality equation is given by 
\begin{align}
Q_{h}^{*}(s,a) & =r_{h}(s,a)+\frac{1}{\tmp}\log\left\{ \E_{s'\sim P_{h}(\cdot\,|\,s,a)}\left[e^{\beta\cdot V_{h+1}^{*}(s')}\right]\right\} ,\label{eq:bellman_opt}\\
V_{h}^{*}(s) & =\max_{a\in\cA}Q_{h}^{*}(s,a),\qquad V_{H+1}^{*}(s)=0.\nonumber 
\end{align}

\paragraph{Exponential Bellman equation.}
We introduce the \textit{exponential Bellman equation}, which is an exponential transformation of Equations \eqref{eq:bellman} and \eqref{eq:bellman_opt} (by taking exponential on both sides): for any policy $\pi$ and tuple $(h,s,a)$, we have
\begin{equation}
e^{\beta\cdot Q_{h}^{\pi}(s,a)}=\E_{s'\sim P_{h}(\cdot\,|\,s,a)} \big[e^{\beta (r_{h}(s,a) + V_{h+1}^{\pi}(s') ) } \big].\label{eq:exp_bellman}
\end{equation}
When $\pi = \pi^*$, we obtain the corresponding optimality equation 
\begin{equation}
e^{\beta\cdot Q_{h}^{*}(s,a)}=\E_{s'\sim P_{h}(\cdot\,|\,s,a)}\big[e^{\beta (r_{h}(s,a) + V_{h+1}^{*}(s') )}\big].\label{eq:exp_bellman_opt}
\end{equation}
Note that Equation \eqref{eq:exp_bellman} associates the current and future cumulative utilities
($Q_{h}^{\pi}$ and $V_{h+1}^{\pi}$) in a multiplicative way. An
implication of 
Equation \eqref{eq:exp_bellman}
is that one may
estimate $e^{\beta\cdot Q_{h}^{\pi}(s,a)}$ by a quantity of the form
\begin{equation}
w_{h}(s,a)=\text{SampAvg}(\{e^{\beta(r_h(s_h, a_h) + V_{h+1}(s_{h+1}))}: (s_h, a_h) = (s,a)\})\label{eq:weight_template}
\end{equation}
given some estimate of the value function $V_{h+1}$. 
Here, 
we denote by $\text{SampAvg}(\cX)$ the sample average computed over
elements in the set $\cX$  throughout past episodes, and it can be seen as an empirical MGF
of cumulative rewards from step $h+1$. 
Equation \eqref{eq:exp_bellman} also suggests the following policy improvement procedure for a risk-sensitive
policy $\pi$: 
\begin{equation}
\pi_{h}(s)\leftarrow\argmax_{a'\in\cA}Q_{h}(s,a')=\begin{cases}
\argmax_{a'\in\cA}e^{\beta\cdot Q_{h}(s,a')}, & \condrew\\
\argmin_{a'\in\cA}e^{\beta\cdot Q_{h}(s,a')}, & \condcost,
\end{cases}\label{eq:exp_policy_improv}
\end{equation}
where $Q_{h}$ denotes some estimated action-value function, possibly obtained from the quantity $w_h$. 

\vspace{1em}

In the next section, we will discuss how the exponential Bellman equation
\eqref{eq:exp_bellman} 
inspires the
development of a novel analytic framework for risk-sensitive RL. Before
proceeding, we introduce a performance metric for the agent. For
each episode $k$, recall that $s_{1}^{k}$ is the initial state chosen
by the environment and let $\pi^{k}$ be the policy of the agent at
the beginning of episode $k$. Then the difference $V_{1}^{*}(s_{1}^{k})-V_{1}^{\pi^{k}}(s_{1}^{k})$
is called the \emph{regret} of the agent in episode $k$. Therefore,
after $K$ episodes, the total regret for the agent is given by
\begin{align}
\reg(K) & \coloneqq\sum_{k\in[K]}[V_{1}^{*}(s_{1}^{k})-V_{1}^{\pi^{k}}(s_{1}^{k})],\label{eq:regret}
\end{align}
which serves as the key performance metric studied in this paper.

\section{Analysis of risk-sensitive RL}

\subsection{Mechanism of existing analysis}

In this section, we provide an informal overview of the mechanism
underlying the existing analysis of risk-sensitive RL. Let us focus on the case $\beta>0$ for simplicity of exposition; similar reasoning holds
for $\beta<0$. A key step in the existing regret analysis of RL
algorithms is to establish a recursion on the difference $V_{h}^{k}-V_{h}^{\pi^{k}}$ over $h\in[H]$,
where $V_{h}^{k}$ is the iterate of an algorithm in step $h$ of
episode $k$ and $V_{h}^{\pi^{k}}$ is the value function of the policy
used in episode $k$. Such approach can be commonly found in the literature
of algorithms that use the upper confidence bound \citep{jin2018q,jin2019provably},
in which the recursion takes the form of 
\begin{equation}
V_{h}^{k}-V_{h}^{\pi^{k}}\le V_{h+1}^{k}-V_{h+1}^{\pi^{k}}+\psi_{h}^{k},\label{eq:recur}
\end{equation}
for $\beta \to 0$ and some quantity $\psi_{h}^{k}$. The work of \citet{fei2020risk},
which studies the risk-sensitive setting under the entropic risk measure,
also follows this approach and derives regret bounds by establishing
the recursion of the form
\begin{equation}
V_{h}^{k}-V_{h}^{\pi^{k}}\le e^{\beta H}\left(V_{h+1}^{k}-V_{h+1}^{\pi^{k}}\right)+\frac{1}{\beta}\tilde{b}_{h}^{k}+e^{\beta H}\tilde{m}_{h}^{k},\label{eq:exp_recur_subopt}
\end{equation}
where $\tilde{b}_{h}^{k}$ denotes the bonus which enforces the upper
confidence bound and leads to the inequality $V_{h}^{k}\ge V_{h}^{\pi}$
for any policy $\pi$, and $\tilde{m}_{h}^{k}$ is part of a martingale
difference sequence. The derivation of Equation \eqref{eq:exp_recur_subopt}
is based on the Bellman equation \eqref{eq:bellman}, which shows
that the action value $Q_{h}^{\pi^{k}}$ is the sum of the reward
$r_{h}$ and the entropic risk measure of $V_{h+1}^{\pi^{k}}$. 
Following \citet{fei2020risk}, we may then unroll the recursion
\eqref{eq:exp_recur_subopt} from $h=H$ to $h=1$ to get 
\begin{equation}
V_{1}^{k}-V_{1}^{\pi^{k}}\le\frac{1}{\beta}e^{\beta H^{2}}\sum_{h}\tilde{b}_{h}^{k}+e^{\beta H^{2}}\sum_{h}\tilde{m}_{h}^{k},\label{eq:exp_recur_sum_subopt}
\end{equation}
given that $V_{H+1}^{k}=V_{H+1}^{\pi^{k}}=0$. Using the inequality
$\reg(K)\le\sum_{k}(V_{1}^{k}-V_{1}^{\pi^{k}})$, $\sum_{k,h}\tilde{b}_{h}^{k}\lesssim(e^{\beta H}-1)\sqrt{K}$
and $\sum_{k,h}\tilde{m}_{h}^{k}\lesssim\sqrt{K}$, we obtain the
regret bound in \citet{fei2020risk} as $\reg(K)\lesssim e^{\beta H^{2}}\frac{e^{\beta H}-1}{\beta H}\sqrt{K}$.
Therefore, it can be seen that the dominating factor $e^{\beta H^{2}}$
in their regret bound originates in Equation \eqref{eq:exp_recur_sum_subopt},
which can be further traced back to the exponential factor $e^{\beta H}$
in the error dynamics \eqref{eq:exp_recur_subopt}.

\subsection{Refined approach via exponential Bellman equation}

While the existing analysis in \eqref{eq:exp_recur_subopt} is motivated
by the Bellman equation of the form given in \eqref{eq:bellman},
we propose to work on the exponential Bellman equation \eqref{eq:exp_bellman}.
Equation \eqref{eq:exp_bellman} operates on the quantities $e^{\beta\cdot Q_{h}^{\pi}}$
and $e^{\beta\cdot V_{h+1}^{\pi}}$, which can be thought of as the
MGFs of the current and future values, while the reward function $r_{h}$
is involved as a multiplicative term. This motivates us to derive
a new recursion: 
\begin{equation}
e^{\beta\cdot V_{h}^{k}}-e^{\beta\cdot V_{h}^{\pi^{k}}}\le e^{\beta\cdot r_{h}^{k}}\big(e^{\beta\cdot V_{h+1}^{k}}-e^{\beta\cdot V_{h+1}^{\pi^{k}}}\big)+b_{h}^{k}+m_{h}^{k},\label{eq:exp_recur}
\end{equation}
where $b_{h}^{k},m_{h}^{k}$ denote some bonus and martingale terms,
respectively, and $r_{h}^{k}$ stands for the reward in step $h$
of episode $k$. Unrolling Equation \eqref{eq:exp_recur} yields 
\begin{equation}
e^{\beta\cdot V_{1}^{k}}-e^{\beta\cdot V_{1}^{\pi^{k}}}\le\sum_{h}e^{\beta\cdot D_{h}^{k}}(b_{h}^{k}+m_{h}^{k}),\label{eq:exp_recur_sum}
\end{equation}
where $D_{h}^{k}=\sum_{i\in[h-1]}r_{i}^{k}$. In words, the error
of $e^{\beta\cdot V_{1}^{k}}-e^{\beta\cdot V_{1}^{\pi^{k}}}$ is bounded
by the weighted sum of bonus and martingale difference terms, where
the weights are given by $e^{\beta\cdot D_{h}^{k}}$, the exponential
rewards up to step $h-1$. We may then apply a localized linearization
of the logarithmic function, which gives $\reg(K)\le\frac{1}{\beta}\sum_{k}(e^{\beta\cdot V_{1}^{k}}-e^{\beta\cdot V_{1}^{\pi^{k}}})$,
and arrives at a regret upper bound (the formal regret bounds will be established in Theorems \ref{thm:regret_V} and \ref{thm:regret_Q_learn} below). Different from Equation \eqref{eq:exp_recur_subopt}
where rewards are only implicitly encoded in $V_{h}^{k}$, in Equation
\eqref{eq:exp_recur} rewards are explicitly involved in the error
dynamics via an exponential term.

To see why Equation \eqref{eq:exp_recur} is intuitively correct,
we may divide both sides of the equation by $\beta$ and take $\beta\to0$.
By doing so, we should expect to obtain quantities from the error
dynamics \eqref{eq:recur} of risk-neutral RL. Since the function
$f_{\beta}(x)=(e^{\beta x}-1)/\beta$ satisfies that $f_{\beta}(x)\to x$
as $\beta\to0$ for any fixed $x$, we have 
\begin{align*}
\lim_{\beta\to0}\frac{1}{\beta}(e^{\beta\cdot V_{h}^{k}}-e^{\beta\cdot V_{h}^{\pi^{k}}}) & =V_{h}^{k}-V_{h}^{\pi^{k}},\\
\lim_{\beta\to0}\frac{1}{\beta}(e^{\beta\cdot r_{h}^{k}}(e^{\beta\cdot V_{h+1}^{k}}-e^{\beta\cdot V_{h+1}^{\pi^{k}}})) & =r_{h}^{k}+V_{h+1}^{k}-(r_{h}^{k}+V_{h+1}^{\pi^{k}})=V_{h+1}^{k}-V_{h+1}^{\pi^{k}},
\end{align*}
recovering terms in \eqref{eq:recur}.
Therefore, the recursion \eqref{eq:exp_recur} can
be seen as generalizing those in the analysis of risk-neutral RL.

By comparing Equations \eqref{eq:exp_recur} and \eqref{eq:exp_recur_subopt},
we see that while both error dynamics are derived from the same underlying
Bellman equation, they inspire drastically different forms of recursion.
Note that the multiplicative
factor $e^{\beta\cdot r_{h}^{k}}$ in Equation \eqref{eq:exp_recur}
is milder than the factor $e^{\beta H}$ in Equation \eqref{eq:exp_recur_subopt}, since $r^k_h \in [0,1]$.
This is the source of an improvement of our refined analysis over
existing works. On the other hand, the success of applying the error
dynamics \eqref{eq:exp_recur} in our analysis crucially depends on
the choice of bonus terms $\{b_{h}^{k}\}$, as an improper choice
would blow up the error $e^{\beta\cdot V_{1}^{k}}-e^{\beta\cdot V_{1}^{\pi^{k}}}$.
This observation motivates our novel bonus design, as we explain next
in Section \ref{sec:algo}. 

\section{Algorithms \label{sec:algo}}

\subsection{Overview of algorithms}

In this section, we propose two model-free algorithms for RL with
the entropic risk measure. We first present RSVI2, which is based on
value iteration, in Algorithm \ref{alg:alg_lsvi}. The algorithm has
two main stages: it first estimates the value function using data
accumulated up to episode $k-1$ (Line \ref{line:LSVI_estim_value_begin}--\ref{line:LSVI_estim_value_end})
and then executes the estimated policy to collect new trajectory (Line
\ref{line LSVI_exec_policy}). In value function estimation, it computes
the weights $w_{h}$, or the empirical MGF of some estimated cumulative rewards
evaluated at $\beta$, which can be seen as a
simple moving average over $\tau \in [k-1]$.
Therefore, Line \ref{line:LSVI_interm_value}
functions as a concrete implementation of Equation \eqref{eq:weight_template}
where the sample average is instantiated as a simple moving average.
Then in Line \ref{line:lsvi_Q_update}, it computes an augmented
estimate $G_{h}$ by combining $w_{h}$ with a bonus term $b_{h}$
(defined in Line \ref{line:lsvi_bonus_def}). This is followed by
thresholding to put $G_{h}$ in the proper range. Note that
$G_{h}$ is an optimistic estimator of the quantity $e^{\beta\cdot Q_{h}^{\pi}}$
in Equation \eqref{eq:exp_bellman}: the construction of $G_{h}$
is augmented by $b_{h}$ so that it encourages exploration of rarely
visited state-action pairs in future episodes, and thereby follows
the principle of Risk-Sensitive Optimism in the Face of Uncertainty \citep{fei2020risk}.
When $\beta<0$, the bonus is subtracted from $w_{h}$,
since a higher level of optimism corresponds to a \emph{smaller} value
of the estimate. In addition, Line \ref{line LSVI_exec_policy} follows
the reasoning of policy improvement suggested in Equation \eqref{eq:exp_policy_improv}.
\begin{algorithm}[t]
	\begin{algorithmic}[1]
		
		\State $Q_{h}(\cdot,\cdot),V_{h}(\cdot)\leftarrow H-h+1$, $N_{h}(\cdot,\cdot)\leftarrow0$
		and $w_{h}(\cdot,\cdot)\leftarrow0$ for all $h\in[H+1]$ 
		
		\For{episode $k=1,\ldots,K$}
		
		\For{step $h=H,\ldots,1$} \label{line:LSVI_estim_value_begin}
		
		\For{$(s,a)\in\cS\times\cA$ such that $N_{h}(s,a)\ge1$}
		
		\State $w_{h}(s,a)\leftarrow\frac{1}{N_{h}(s,a)}\sum_{\tau\in[k-1]}\indic\{(s_{h}^{\tau},a_{h}^{\tau})=(s,a)\}\cdot e^{\beta[r_{h}(s,a)+V_{h+1}(s_{h+1}^{\tau})]}$
		\label{line:LSVI_interm_value}
		
		\State $b_{h}(s,a)\leftarrow c|e^{\beta(H-h+1)}-1|\sqrt{\frac{S\log(HSAK/\delta)}{N_{h}(s,a)}}$
		where $c>0$ is a universal constant    \label{line:lsvi_bonus_def}
		
		\State $G_{h}(s,a)\leftarrow\begin{cases}
		\min\{\w_{h}(s,a)+b_{h}(s,a),\ e^{\beta(H-h+1)}\}, & \condrew\\
		\max\{\w_{h}(s,a)-b_{h}(s,a),\ e^{\beta(H-h+1)}\}, & \condcost
		\end{cases}$ \label{line:lsvi_Q_update}
		
		\State $V_{h}(s)\leftarrow\max_{a'\in\cA}\frac{1}{\beta}\log\{G_{h}(s,a')\}$
		\label{line:LSVI_value}
		
		\EndFor 
		
		\EndFor \label{line:LSVI_estim_value_end}

		\State $\forall h\in[H]$, take		$a_{h}\leftarrow\argmax_{a'\in\cA}\frac{1}{\beta} \log\{G_{h}(s_{h},a')\}$;
		observe $r_{h}(s_{h},a_{h})$, $s_{h+1}$ 
		\label{line LSVI_exec_policy}
		
		\State Add 1 to $N_{h}(s_{h},a_{h})$
		
		\EndFor
		
	\end{algorithmic}
	
	\caption{RSVI2 \label{alg:alg_lsvi}}
\end{algorithm}

Next, we introduce RSQ2 in Algorithm \ref{alg:Q_learn}, which is
based on Q-learning. Similar to Algorithm \ref{alg:alg_lsvi}, it
consists of value estimation (Line \ref{line:qlearn_bonus_def}--\ref{line:qlearn_value})
and policy execution (Line \ref{line:qlearn_exec_policy}) steps.
By combining Lines \ref{line:qlearn_interm_value} and \ref{line:qlearn_Q_update},
we see that Algorithm \ref{alg:Q_learn} computes the optimistic estimate $G_{h}$
as a projection of an exponential moving average of empirical MGFs:
\begin{align}
G_{h}(s_{h},a_{h}) & \leftarrow\Pi_{h} \{ \text{EMA}(\{e^{\beta[r_{h}(s_{h},a_{h})+V_{h+1}(s_{h+1})]}\}) \},
\label{eq:qlearn_EMA}
\end{align}
where $\Pi_{h}$ denotes a projection that depends on step $h$. In
particular, Line \ref{line:qlearn_interm_value} can be interpreted as a
computation of empirical MGFs evaluated at $\beta$ and thus a concrete
implementation of Equation \eqref{eq:weight_template} using an exponential
moving average. This is in contrast with the simple moving average
update in Algorithm \ref{alg:alg_lsvi}.  
\begin{algorithm}
	\begin{algorithmic}[1]
		
		\State $Q_{h}(\cdot,\cdot),V_{h}(\cdot)\leftarrow H-h+1$ $\condrew;$
		$Q_{h}(\cdot,\cdot),V_{h}(\cdot)\leftarrow0$ otherwise, for all $h\in[H+1]$
		
		\State $N_{h}(\cdot,\cdot)\leftarrow0$ for all $h\in[H]$; $\alpha_{u}\leftarrow\frac{H+1}{H+u}$
		for $u\in\mathbb{Z}$
		
		\For{episode $k=1,\ldots,K$}
		
		\State Receive the initial state $s_{1}$
		
		\For{step $h=1,\ldots,H$} \label{line:qlearn_estim_value_begin}
		
		\State Take action $a_{h}\leftarrow\argmax_{a'\in\cA}\frac{1}{\beta}\log\{G_{h}(s_{h},a')\}$ and observe $r_{h}(s_{h},a_{h})$ and
		$s_{h+1}$\label{line:qlearn_exec_policy}

		\State Add 1 to $N_{h}(s_{h},a_{h})$; $\ \ $
		$t\leftarrow N_{h}(s_{h},a_{h})$
		
		\State $b_{h,t}\leftarrow c|e^{\beta(H-h+1)}-1|\sqrt{\frac{H\log(HSAK/\delta)}{t}}$
		for some universal constant $c>0$ \label{line:qlearn_bonus_def}
		
		\State $w_{h}(s_{h},a_{h})\leftarrow(1-\alpha_{t})\cdot G_{h}(s_{h},a_{h})+\alpha_{t}\cdot e^{\beta[r_{h}(s_{h},a_{h})+V_{h+1}(s_{h+1})]}$
		\label{line:qlearn_interm_value}
		
		\State $G_{h}(s_{h},a_{h})\leftarrow\begin{cases}
		\min\{\w_{h}(s_{h},a_{h})+\alpha_{t}b_{h,t},\ e^{\beta(H-h+1)}\}, & \text{if }\beta>0\\
		\max\{\w_{h}(s_{h},a_{h})-\alpha_{t}b_{h,t},\ e^{\beta(H-h+1)}\}, & \text{if }\beta<0
		\end{cases}$ \label{line:qlearn_Q_update}
		
		\State $V_{h}(s_{h})\leftarrow\max_{a'\in\cA}\frac{1}{\beta}\log\{G_{h}(s_{h},a')\}$\label{line:qlearn_value}
		
		\EndFor \label{line:qlearn_estim_value_end}
		
		\EndFor
		
	\end{algorithmic}
	
	\caption{RSQ2 \label{alg:Q_learn}}
\end{algorithm}

Although Algorithms \ref{alg:alg_lsvi} and \ref{alg:Q_learn} are inspired
by RSVI and RSQ of \citet{fei2020risk}, respectively, we note that the main novelty
of our algorithms lies in the bonus terms ($b_{h}$ in Algorithm
\ref{alg:alg_lsvi} and $b_{h,t}$ in Algorithm \ref{alg:Q_learn}),
which we call the \emph{doubly decaying} bonus. We discuss this new
bonus design in the following.

\subsection{Doubly decaying bonus}

Let us focus on $\beta > 0$  for this discussion. 
In optimism-based algorithms, the bonus term is used to enforce
the upper confidence bound in order to encourage sufficient exploration
in uncertain environments. It takes the form of a multiplier times
a factor that is inversely proportional to visit counts $\{N_h\}$. Our bonus
follows this structure and is given by 
\begin{equation}
b_{h}(s,a)\propto(e^{\beta(H-h+1)}-1)\sqrt{\frac{1}{N_{h}(s,a)}},\label{eq:bonus_ours}
\end{equation}
ignoring factors that do not vary in $(h,s,a)$. 
In Equation \eqref{eq:bonus_ours}, the
quantity $e^{\beta(H-h+1)}$ plays the role of the multiplier and
$\sqrt{1/N_{h}(s,a)}$
is the factor that decreases in
the visit count.  While the component 
$\sqrt{1/N_{h}(s,a)}$
is common in bonus terms, our new bonus is designed to shrink its
multiplier deterministically and exponentially across the horizon
steps, as $e^{\beta(H-h+1)}-1$ decreases from $e^{\beta H}-1$ in
step $h=1$ to $e^{\beta}-1$ in step $h=H$.  This is in sharp contrast
with the bonus terms typically found in risk-neutral RL algorithms,
where the multipliers are kept constant in  $h$ (usually as a
constant multiple of $H$). Furthermore, our bonus design is also in contrast
with that in RSVI and RSQ proposed by \citet{fei2020risk},
whose multiplier is $e^{\beta H}-1$ and kept fixed along the horizon.
Because $b_h$ decays both in the visit count $N_{h}(s,a)$
(across episodes) and the multiplier $e^{\beta(H-h+1)}-1$ (across
the horizon), we name it as \emph{doubly decaying} bonus. We
remark that this is a novel feature of Algorithms \ref{alg:alg_lsvi}
and \ref{alg:Q_learn}, compared to RSVI and RSQ. Let us discuss how
this new exploration mechanism is motivated from the error dynamics
\eqref{eq:exp_recur_sum}.

\paragraph{Motivation of exponential decay.}

From Equation \eqref{eq:exp_recur_sum}, we see that the error of
the iterate is bounded by the sum of weighted bonus terms, where
the weights are of the form $e^{\beta\cdot D_{h}}$ and $D_{h}\in[0,h-1]$.
Choosing $b_{h}\propto e^{\beta(H-h+1)}-1$ ensures that the weighted
bonus is on the order of $e^{\beta H}-1$ at maximum. On the other
hand, if we use the bonus as in \citet{fei2020risk}, which is proportional
to $e^{\beta H}-1$, then we would end up with a multiplicative factor
$e^{2\beta H}-1$ in regret, which is exponentially larger than $e^{\beta H}-1$. 
%
An alternative way to understanding the exponential decay of our bonus
is as follows. At step $h$, the estimated value function is $V_{h}\in[0,H-h+1]$,
which implies $e^{\beta\cdot V_{h}}\in[1,e^{\beta(H-h+1)}]$. The
iterate $G_{h}$ (of Algorithm \ref{alg:alg_lsvi} or \ref{alg:Q_learn})
is used to estimate $e^{\beta\cdot Q_{h}^{\pi}}$, with its estimation
error given by
\[
|e^{\beta\cdot Q_{h}^{\pi}}-G_{h}|\approx|e^{\beta\cdot Q_{h}^{\pi}}-\widehat{\P}_{h}e^{\beta(r_{h}+V_{h+1})}|\le e^{\beta(H-h+1)}-1,
\]
where $\widehat{\P}_{h}$ denotes an empirical average operator over
historical data in step $h$. Therefore, the estimation error of $G_{h}$
shrinks exponentially across the horizon. Since  bonus is used to compensate for and dominate the estimation
error, the minimal order of $b_{h}$ required is thus $e^{\beta(H-h+1)}-1$,
which is exactly the multiplier in Equation \eqref{eq:bonus_ours}.


As a passing note, we remark that the decaying multiplier is not necessary
in risk-neutral RL algorithms, since the estimation error therein
satisfies $|Q_{h}-\widehat{\P}_{h}(r_{h}+V_{h+1})|\le H-h+1$, which
is upper bounded by $H$ for all $h\in[H]$. This implies that it
suffices to simply set the bonus multiplier as a constant multiple
of $H$. In contrast, as we have explained, the estimation error of
our algorithms decays exponentially in step $h$, and an adaptive
and exponentially decaying bonus is needed.

\paragraph{Comparison with Bernstein-type bonus.}

We also compare our bonus in Equation \eqref{eq:bonus_ours} with
the Bernstein-type bonus commonly used to improve sample efficiency
of  risk-neutral RL algorithms \citep{azar2017minimax,jin2018q}.
The Bernstein-type bonus takes the form of 
\begin{equation}
\bar{b}_{h}(s,a)\propto\sqrt{\frac{H+\widehat{\Var}(V_{h+1})}{N_{h}(s,a)}}+o\left(\sqrt{\frac{1}{N_{h}(s,a)}}\right),\label{eq:bonus_bern}
\end{equation}
where $\widehat{\Var}(\cdot)$ denotes an empirical variance operator
over historical data  and $o(\cdot)$ denotes a vanishing
term as $N_{h}(s,a)\to\infty$. Our bonus in Equation \eqref{eq:bonus_ours}
is different from the Bernstein-type bonus in Equation \eqref{eq:bonus_bern}
in mechanism: our bonus features
the multiplier $e^{\beta(H-h+1)}-1$ which decays exponentially and
deterministically over $h\in[H]$, whereas the Bernstein-type bonus
uses $\sqrt{H+\widehat{\Var}(V_{h+1})}$ as the multiplier (ignoring the vanishing term).
The term $\widehat{\Var}(V_{h+1})$ depends on
the trajectory of the learning process. Therefore the multiplier is stochastic
and stays on the polynomial order
of $H$ across the horizon. Moreover, it is unclear how the multiplier behaves in terms of step $h$.

\section{Main results}

In this section, we present and discuss our main theoretical results
for Algorithms \ref{alg:alg_lsvi} and \ref{alg:Q_learn}.

\begin{thm}
	\label{thm:regret_V}For any $\delta\in(0,1]$, with probability
	at least $1-\delta$ there exists a universal constant $c>0$ (used
	in Algorithm \ref{alg:alg_lsvi}), such that the regret of Algorithm
	\ref{alg:alg_lsvi} is bounded by 
	\begin{align*}
	\reg(K) & \lesssim\frac{e^{|\beta| H}-1}{|\beta| H}\sqrt{H^{4}S^{2}AK\log^{2}(HSAK/\delta)}.
	\end{align*}
\end{thm}

\begin{thm}
	\label{thm:regret_Q_learn}For any $\delta\in(0,1]$, with probability
	at least $1-\delta$ and when $K$ is sufficiently large, there exists
	a universal constant $c>0$ (used in Algorithm \ref{alg:Q_learn})
	such that the regret of Algorithm \ref{alg:Q_learn} obeys
	\begin{align*}
	\reg(K) & \lesssim\frac{e^{|\beta| H}-1}{|\beta| H}\sqrt{H^{3}SAK\log(HSAK/\delta)}.
	\end{align*}
\end{thm}
The proof of the two theorems are provided in Appendices \ref{sec:proof_regret_V}
and \ref{sec:proof_regret_Q_learn}, respectively. Note that the above
results generalize those in the literature of risk-neutral RL: when
$\beta\to0$, we recover the same regret bounds of LSVI in \citet{jin2019provably}
and Q-learning in \citet{jin2018q}. 

Let us discuss the connections between our results and those in \citet{fei2020risk}.
The work of \citet{fei2020risk} proposes two algorithms, RSVI and
RSQ, that attain the regret bound 
\begin{equation}
\reg(K)\lesssim e^{|\beta| H^{2}}\cdot\frac{e^{|\beta| H}-1}{|\beta| H}\sqrt{\text{poly}(H)\cdot K},\label{eq:upper_old_informal}
\end{equation}
and a lower bound incurred by any algorithm 
\begin{equation}
\reg(K)\gtrsim\frac{e^{|\beta| H'}-1}{|\beta| H}\sqrt{\text{poly}(H)\cdot K},\label{eq:lower_informal}
\end{equation}
where $H'$ is a linear function in $H$; for simplicity of presentation, we exclude polynomial dependencies on other parameters and logarithmic
factors from the two bounds. In particular, the proof
of the lower bound is based on reducing an hard instance of MDP to
a multi-armed bandit. It is a priori unclear whether the extra exponential
factor $e^{|\beta| H^{2}}$ in the upper bound \eqref{eq:upper_old_informal}
is fundamental in the MDP setting, or is due to suboptimal analysis
or algorithmic design. We would like to mention that although one
trivial way of avoiding the $e^{|\beta| H^{2}}$ factor in the upper
bound \eqref{eq:upper_old_informal} is to use a sufficiently small
$|\beta|$ in the algorithms of \citet{fei2020risk} (e.g., $|\beta|\le\frac{1}{H^{2}}$
so that $e^{|\beta| H^{2}}\lesssim1$), such a small $|\beta|$ defeats
the very purpose of have an appropriate degree of risk-sensitivity
in the algorithms. Hence, an answer for \emph{all }$\beta\neq0$
would be desirable. 

In view of Theorems \ref{thm:regret_V} and \ref{thm:regret_Q_learn},
we see that our Algorithms \ref{alg:alg_lsvi} and \ref{alg:Q_learn}
achieve regret bounds that are exponentially sharper than those of
RSVI and RSQ. In particular, our results eliminate the $e^{|\beta| H^{2}}$
factor from Equation \eqref{eq:upper_old_informal} thanks to the novel analysis and doubly
decaying bonus in our algorithms, which are inspired by the exponential Bellman equation \eqref{eq:exp_bellman}.  As a result, our
bounds significantly narrow the gap between upper bounds and
the lower bound \eqref{eq:lower_informal}.

\section*{Acknowledgments}
	Z.$\ $Yang acknowledges Simons Institute (Theory of Reinforcement Learning). 
	Y.\ Chen is partially supported by NSF grant CCF-1704828 and CAREER Award CCF-2047910.
	Z.$\ $Wang acknowledges National Science Foundation (Awards 2048075, 2008827, 2015568, 1934931), Simons Institute (Theory of Reinforcement Learning), Amazon, J.P. Morgan, and Two Sigma for their supports.

\nocite{}
\bibliographystyle{plainnat}
\bibliography{references}


\newpage

\appendix
\appendixpage

\section{Connections to distribution RL  \label{sec:dist_RL}}

In this appendix, we establish connections between risk-sensitive RL and distributional RL via the lens of the exponential Bellman equation. 

Distributional RL has been studied in the line of works \citep{bellemare2017distributional,dabney2018distributional,mavrin2019distributional,morimura2010nonparametric,morimura2012parametric,rowland2018analysis, farahmand2019value, yang2019fully}.
The framework of distributional RL is built upon the following key equation, namely
the distributional Bellman equation:
\begin{equation}
\forall h\in[H],\quad Z_{h}^{\pi}(s,a)\overset{d}{=}R_{h}(s,a)+Z_{h+1}^{\pi}(X',U'),\label{eq:dist_bellman}
\end{equation}
for a fixed policy $\pi$, where $Z_{H+1}^{\pi}(\cdot,\cdot) \coloneqq 0$, $X'\sim P_{h}(\cdot\ |\ s,a)$, $U'\sim\pi(\cdot\ |\ X')$
and $R_{h}$ is the reward distribution in step $h$. Here, we use
$\overset{d}{=}$ to denote equality in distribution. It can be seen
that $Z_{h}^{\pi}(s,a)$ is the distribution of cumulative rewards
under policy $\pi$ at step $h$, when the state and action $(s,a)$
are visited in step $h$. Based on Equation \eqref{eq:dist_bellman},
a distributional Bellman optimality operator $\T_h$ is given by 
\begin{equation}
[\T_h Z](s,a)\overset{d}{\coloneqq}R_{h}(s,a)+Z_{h+1}(X',\argmax_{a'\in\cA}\E[Z_{h+1}(X',a')]),\label{eq:dist_bellman_opt}
\end{equation}
where again $X'\sim P_{h}(\cdot\ |\ s,a)$. Note that in Equation
\eqref{eq:dist_bellman_opt}, the optimal action is greedy with respect
to the \emph{expectation} of the distribution $Z_{h+1}$. 
Most existing distributional
RL algorithms work with distribution estimates such
as quantiles \citep{dabney2018distributional,dabney2018implicit}
or empirical distribution functions \citep{bellemare2017distributional,rowland2018analysis}.

Now recall the exponential Bellman equation \eqref{eq:exp_bellman}, which takes the form
\begin{equation}
\forall h\in[H],\quad e^{\mu\cdot Q_{h}^{\pi}(s,a)}= \E_{X'}[e^{\mu  (r_{h}(s,a) + V_{h+1}^{\pi}(X') )}],\label{eq:exp_bellman_dist}
\end{equation}
for any fixed $\mu \in \real$, where $V_{H+1}^{\pi}(\cdot) \coloneqq 0$, $X'\sim P_{h}(\cdot\ |\ s,a)$ and $r_h$ is the deterministic reward function by our assumption. Given the definitions \eqref{eq:value_func} and \eqref{eq:action_value_func} with $\beta$ replaced by $\mu$, we note that both $Q_{h}^{\pi}$ and $V_{h+1}^{\pi}$ in the above equation depend on the value of $\mu$ (which we omit for simplicity of notations). Then by the definition of $Q_{h}^{\pi}$ in Equation \eqref{eq:action_value_func}, one sees that $\{e^{\mu \cdot Q_{h}^{\pi}}: \mu\in\real\}$ represents the MGF of the cumulative rewards at step $h$ when policy $\pi$ is executed. Hence, the exponential Bellman equation for risk-sensitive RL provides an instantiation of Equation \eqref{eq:dist_bellman} through the MGF of rewards.

\section{Proof of Theorem \ref{thm:regret_V} \label{sec:proof_regret_V}}

First, we set some notations and definitions. Define $\iota\coloneqq\log(2HSAK/\delta)$
for a given $\delta\in(0,1]$. We adopt the shorthands $\indic_{h}^{\tau}(s,a)\coloneqq\indic\{(s_{h}^{\tau},a_{h}^{\tau})=(s,a)\}$
and  $r_{h}^{\tau}\coloneqq r_{h}(s_{h}^{\tau},a_{h}^{\tau})$ for
$(\tau,h)\in[K]\times[H]$. We let $N_{h}^{k}(s,a)$ be the visit
count of $(h,s,a)$ at the beginning of episode $k$. We denote by
$V_{h}^{k}$, $G_{h}^{k}$, $b_{h}^{k}$ the values of $V_{h}$, $G_{h}$,
$b_{h}$ after the updates in step $h$ of episode $k$, respectively.
We also set $Q_{h}^{k}=\frac{1}{\beta} \log\{G_{h}^{k}\}$.

For the time being we consider $\beta>0$. For $h\in[H]$, we define
\begin{align*}
\delta_{h}^{k} & \coloneqq e^{\beta V_{h}^{k}(s_{h}^{k})}-e^{\beta V_{h}^{\pi^{k}}(s_{h}^{k})},\\
\zeta_{h+1}^{k} & \coloneqq[P_{h}(e^{\beta[r_{h}(s_{h}^{k},a_{h}^{k})+V_{h+1}^{k}(s')]}-e^{\beta[r_{h}(s_{h}^{k},a_{h}^{k})+V_{h+1}^{\pi^{k}}(s')]})](s_{h}^{k},a_{h}^{k})-e^{\beta r_{h}(s_{h}^{k},a_{h}^{k})}\delta_{h+1}^{k},
\end{align*}
where $[P_{h}f](s,a)\coloneqq\E_{s'\sim P_{h}(\cdot|s,a)}[f(s')]$
for any $f:\cS\to\real$ and $(s,a)\in\cS\times\cA$. It can be seen
that $b_{h}^{k}$ in Algorithm \ref{alg:alg_lsvi} can be equivalently
defined as 
\begin{equation}
b_{h}^{k}\coloneqq c(e^{\beta(H-h+1)}-1)\sqrt{\frac{S\iota}{\max\{1,N_{h}^{k}(s_{h}^{k},a_{h}^{k})\}}},\label{eq:v_gamma}
\end{equation}
where $c$ is the universal constant from Lemma \ref{lem:v_G1_bound}. For
any $(k,h)\in[K]\times[H]$, we have 
\begin{align}
\delta_{h}^{k} & \overset{(i)}{=}(e^{\beta\cdot Q_{h}^{k}}-e^{\beta\cdot Q_{h}^{\pi^{k}}})(s_{h}^{k},a_{h}^{k})\nonumber \\
& \overset{(ii)}{=}\left[\min\{e^{\beta(H-h+1)},(w_{h}^{k}+b_{h}^{k})(s_{h}^{k},a_{h}^{k})\}-\E_{s'\sim P_{h}(\cdot\,|\,s_{h}^{k},a_{h}^{k})}e^{\beta[r_{h}(s_{h}^{k},a_{h}^{k})+V_{h+1}^{k}(s')]}\right]\nonumber \\
& \quad+\left[\E_{s'\sim P_{h}(\cdot\,|\,s_{h}^{k},a_{h}^{k})}e^{\beta[r_{h}(s_{h}^{k},a_{h}^{k})+V_{h+1}^{k}(s')]}-\E_{s'\sim P_{h}(\cdot\,|\,s_{h}^{k},a_{h}^{k})}e^{\beta[r_{h}(s_{h}^{k},a_{h}^{k})+V_{h+1}^{\pi^{k}}(s')]}\right]\nonumber \\
& \overset{(iii)}{\le}2b_{h}^{k}+[P_{h}(e^{\beta[r_{h}(s_{h}^{k},a_{h}^{k})+V_{h+1}^{k}(s')]}-e^{\beta[r_{h}(s_{h}^{k},a_{h}^{k})+V_{h+1}^{\pi^{k}}(s')]})](s_{h}^{k},a_{h}^{k})\nonumber \\
& =2b_{h}^{k}+e^{\beta\cdot r_{h}(s_{h}^{k},a_{h}^{k})}\delta_{h+1}^{k}+\zeta_{h+1}^{k}\label{eq:lsvi_delta_recursion}
\end{align}
In the above equation, step $(i)$ holds by the construction of Algorithm
\ref{alg:alg_lsvi} and the definition of $V_{h}^{\pi^{k}}$ in Equation
\eqref{eq:bellman}; step $(ii)$ holds by Equations \eqref{eq:vi_q2_simp}
and \eqref{eq:vi_q3_simp}. step $(iii)$ holds on the event of Lemma
\ref{lem:v_G1_bound};  the last step follows from Lemma \ref{lem:v_Vk_dom}. 

Using the fact that $V_{H+1}^{k}(s)=V_{H+1}^{\pi^{k}}(s)=0$ and
that $r_{h}(\cdot,\cdot)\in[0,1]$, we can expand the recursion in
Equation \eqref{eq:lsvi_delta_recursion} and get 
\begin{align*}
\delta_{1}^{k} & \le\sum_{h\in[H]}e^{\beta(h-1)}\zeta_{h+1}^{k}+2\sum_{h\in[H]}e^{\beta(h-1)}b_{h}^{k}.
\end{align*}
Summing the above display over $k\in[K]$ gives 
\begin{equation}
\sum_{k\in[K]}\delta_{1}^{k}\le\sum_{k\in[K]}\sum_{h\in[H]}e^{\beta(h-1)}\zeta_{h+1}^{k}+2\sum_{k\in[K]}\sum_{h\in[H]}e^{\beta(h-1)}b_{h}^{k}\label{eq:lsvi_delta1}
\end{equation}

Let us now control the two terms in Equation \eqref{eq:lsvi_delta1}.
Note that $\{\zeta_{h+1}^{k}\}$ is a martingale difference sequence
satisfying $\left|\zeta_{h}^{k}\right|\le2H$ for all $(k,h)\in[K]\times[H]$.
By the Azuma-Hoeffding inequality, we have for any $t>0$, 
\[
\P\left(\sum_{k\in[K]}\sum_{h\in[H]}e^{\beta(h-1)}\zeta_{h+1}^{k}\ge t\right)\le\exp\left(-\frac{t^{2}}{2HK(e^{\beta H}-1)^{2}}\right).
\]
Hence, with probability $1-\delta/2$, there holds 
\begin{equation}
\sum_{k\in[K]}\sum_{h\in[H]}e^{\beta(h-1)}\zeta_{h+1}^{k}\le(e^{\beta H}-1)\sqrt{2HK\log(2/\delta)}\le(e^{\beta H}-1)\sqrt{2HK\iota},\label{eq:lsvi_mtg_bound}
\end{equation}
where $\iota=\log(2HSAK/\delta)$. For the second term in Equation
\eqref{eq:lsvi_delta1}, recall the definition of $b_{h}^{k}$ in
Equation \eqref{eq:v_gamma}, and we can derive 
\begin{align}
\sum_{k\in[K]}\sum_{h\in[H]}e^{\beta(h-1)}b_{h}^{k} & \le\sum_{k\in[K]}\sum_{h\in[H]}c(e^{\beta H}-1)\sqrt{S\iota}\sqrt{\frac{1}{\max\{1,N_{h}^{k}(s_{h}^{k},a_{h}^{k})\}}}\nonumber \\
& =c(e^{\beta H}-1)\sqrt{S\iota}\sum_{k\in[K]}\sum_{h\in[H]}\sqrt{\frac{1}{\max\{1,N_{h}^{k}(s_{h}^{k},a_{h}^{k})\}}}\nonumber \\
& \overset{(i)}{\le}c(e^{\beta H}-1)\sqrt{S\iota}\sum_{h\in[H]}\sqrt{K}\sqrt{\sum_{k\in[K]}\frac{1}{\max\{1,N_{h}^{k}(s_{h}^{k},a_{h}^{k})\}}}\nonumber \\
& \le c(e^{\beta H}-1)\sqrt{S\iota}\sqrt{2H^{2}SAK\iota},\label{eq:lsvi_sum_quad_Lambda_bound}
\end{align}
where step $(i)$ follows the Cauchy-Schwarz inequality and the last
step holds by the pigeonhole principle. Plugging Equations \eqref{eq:lsvi_mtg_bound}
and \eqref{eq:lsvi_sum_quad_Lambda_bound} back to Equation \eqref{eq:lsvi_delta1}
yields 
\begin{align*}
\sum_{k\in[K]}\delta_{1}^{k} & \le(e^{\beta H}-1)\sqrt{2HK\iota}+2c(e^{\beta H}-1)\sqrt{2H^{2}S^{2}AK\iota^{2}}\\
& \lesssim(e^{\beta H}-1)\sqrt{2H^{2}S^{2}AK\iota^{2}},
\end{align*}
The proof for $\beta>0$ is completed by invoking Lemma \ref{lem:reg_decomp}
on the event of Lemma \ref{lem:v_Vk_dom}. We note that the proof
of $\beta<0$ follows a similar procedure and is therefore omitted.

\subsection{Auxiliary lemmas}

Let us fix a pair $(s,a)\in\cS\times\cA$. Recall from Algorithm \ref{alg:alg_lsvi}
that 
\begin{equation}
w_{h}^{k}(s,a)=\frac{1}{N_{h}^{k}(s,a)}\sum_{\tau\in[k-1]}\indic_{h}^{\tau}(s,a)\left[e^{\beta[r_{h}^{\tau}+V_{h+1}^{k}(s_{h+1}^{\tau})]}\right].\label{eq:lsvi_weights_equiv}
\end{equation}
If $N_{h}^{k}(s,a)\ge1$, we define 
\begin{align*}
q_{h,1}^{k,+}(s,a) & \coloneqq\begin{cases}
\w_{h}^{k}(s,a)+b_{h}^{k}(s,a), & \condrew,\\
\w_{h}^{k}(s,a)-b_{h}^{k}(s,a), & \condcost.
\end{cases}\\
q_{h,1}^{k}(s,a) & \coloneqq\begin{cases}
\min\{q_{h,1}^{k,+}(s,a),e^{\beta(H-h+1)}\}, & \condrew,\\
\max\{q_{h,1}^{k,+}(s,a),e^{\beta(H-h+1)}\}, & \condcost,
\end{cases}
\end{align*}
and if $N_{h}^{k}(s,a)=0$, we let 
\[
q_{h,1}^{k,+}(s,a)=q_{h,1}^{k}(s,a)\coloneqq e^{\beta(H-h+1)}.
\]
Also define 
\[
q_{h,2}^{k}(s,a)\coloneqq\begin{cases}
\frac{1}{N_{h}^{k}(s,a)}\sum_{\tau\in[k-1]}\indic_{h}^{\tau}(s,a)\left[\E_{s'\sim P_{h}(\cdot\,|\,s_{h}^{\tau},a_{h}^{\tau})}e^{\beta[r_{h}^{\tau}+V_{h+1}^{k}(s')]}\right], & \text{if }N_{h}^{k}(s,a)\ge1,\\
e^{\beta(H-h+1)}\ \ \condrew;\qquad 1\ \ \condcost, & \text{if }N_{h}^{k}(s,a)=0,
\end{cases}
\]
and for any policy $\pi$, 
\[
q_{h,3}^{k,\pi}(s,a)\coloneqq\begin{cases}
\frac{1}{N_{h}^{k}(s,a)}\sum_{\tau\in[k-1]}\indic_{h}^{\tau}(s,a)\left[\E_{s'\sim P_{h}(\cdot\,|\,s_{h}^{\tau},a_{h}^{\tau})}e^{\beta[r_{h}^{\tau}+V_{h+1}^{\pi}(s')]}\right], & \text{if }N_{h}^{k}(s,a)\ge1,\\
e^{\beta \cdot Q_{h}^{\pi}(s,a)} & \text{if }N_{h}^{k}(s,a)=0.
\end{cases}
\]
It can be seen that 
\begin{equation}
q_{h,2}^{k}(s,a)=\E_{s'\sim P_{h}(\cdot\,|\,s,a)}e^{\beta[r_{h}(s,a)+V_{h+1}^{k}(s')]}\label{eq:vi_q2_simp}
\end{equation}
when $N_{h}^{k}(s,a)\ge1$, and 
\begin{equation}
q_{h,3}^{k,\pi}(s,a)=e^{\beta Q_{h}^{\pi}(s,a)}=\E_{s'\sim P_{h}(\cdot\,|\,s,a)}e^{\beta[r_{h}(s,a)+V_{h+1}^{\pi}(s')]}\label{eq:vi_q3_simp}
\end{equation}
for all $(k,h,s,a)\in[K]\times[H]\times\cS\times\cA$ by the exponential
Bellman equation \eqref{eq:exp_bellman}. We have that $\condrew$,
\begin{equation}
e^{\beta Q_{h}^{k}}-e^{\beta Q_{h}^{\pi}}=q_{h,1}^{k}-q_{h,3}^{k,\pi}=(q_{h,1}^{k}-q_{h,2}^{k})+(q_{h,2}^{k}-q_{h,3}^{k,\pi}),\label{eq:v_G1_G2_sum_pos}
\end{equation}
and $\condcost$, 
\begin{equation}
e^{\beta Q_{h}^{\pi}}-e^{\beta Q_{h}^{k}}=q_{h,3}^{k,\pi}-q_{h,1}^{k}=(q_{h,3}^{k,\pi}-q_{h,2}^{k})+(q_{h,2}^{k}-q_{h,1}^{k}).\label{eq:v_G1_G2_sum_neg}
\end{equation}

Let us state a uniform concentration result. 
\begin{lem}
	\label{lem:v_unif_concen}Define $\iota\coloneqq\log(2HSAK/\delta)$
	and 
	\[
	\bar{\cV}_{h+1}\coloneqq\left\{ \bar{V}_{h+1}:\cS\to\real\mid\forall s\in\cS,\ \bar{V}_{h+1}(s)\in[0,H-h]\right\} .
	\]
	For any $\delta\in(0,1]$, there extsts a universal constant $c_{0}>0$
	such that with probability $1-\delta$, we have 
	\begin{align*}
	& \quad\left|\frac{1}{N_{h}^{k}(s,a)}\sum_{\tau\in[k-1]}\indic_{h}^{\tau}(s,a)\left[e^{\beta[r_{h}^{\tau}+\bar{V}(s_{h+1}^{\tau})]}-\E_{s'\sim P_{h}(\cdot\,|\,s_{h}^{\tau},a_{h}^{\tau})}e^{\beta[r_{h}^{\tau}+\bar{V}(s')]}\right]\right|\\
	& \le c_{0}(e^{\beta(H-h+1)}-1)\sqrt{\frac{S\iota}{\max\{1,N_{h}^{k}(s,a)\}}}
	\end{align*}
	for all $\bar{V}\in\bar{\cV}_{h+1}$ and all $(k,h,s,a)\in[K]\times[H]\times\cS\times\cA$
	that satisfies $N_{h}^{k}(s,a)\ge1$.
\end{lem}
\begin{proof}
	The result is a simple adaptation of \citep[Lemma 6]{fei2020risk}.
\end{proof}
We now control the difference $q_{h,1}^{k}-q_{h,2}^{k}$.
\begin{lem}
	\label{lem:v_G1_bound} Recall the definition of $b_{h}^{k}$ from
	Algorithm \ref{alg:alg_lsvi}. For all $(k,h,s,a)\in[K]\times[H]\times\cS\times\cA$,
	there exists some universal constant $c>0$ (where $c$
	is used in Line \ref{line:lsvi_bonus_def} of Algorithm \ref{alg:alg_lsvi})
	such that the following holds with probability at least $1-\delta/2$:
	if $\beta>0$, we have
	\[
	0\le(q_{h,1}^{k}-q_{h,2}^{k})(s,a)\le2b_{h}^{k},
	\]
	and if $\beta<0$, we have 
	\[
	0\le(q_{h,2}^{k}-q_{h,1}^{k})(s,a)\le2b_{h}^{k}.
	\]
\end{lem}
\begin{proof}
	Let us fix a tuple $(k,h,s,a)\in[K]\times[H]\times\cS\times\cA$.
	
	\textbf{Case $\beta>0$. }For $N_{h}^{k}(s,a)=0$, we have $q_{h,1}^{k}\le e^{\beta(H-h+1)}$
	and $q_{h,2}^{k}\ge1$ by construction and the result follows immediately.
	Now we assume $N_{h}^{k}(s,a)\ge1$. By Equation \eqref{eq:lsvi_weights_equiv}
	we can compute 
	\begin{align*}
	& \quad\left|(q_{h,1}^{k,+}-b_{h}^{k}-q_{h,2}^{k})(s,a)\right|\\
	& =\left|\frac{1}{N_{h}^{k}(s,a)}\sum_{\tau\in[k-1]}\indic_{h}^{\tau}(s,a)\left[e^{\beta[r_{h}^{\tau}+V_{h+1}^{k}(s_{h+1}^{\tau})]}-\E_{s'\sim P_{h}(\cdot\,|\,s_{h}^{\tau},a_{h}^{\tau})}[e^{\beta[r_{h}^{\tau}+V_{h+1}^{k}(s')]}]\right]\right|\\
	& \le c_{0}(e^{\beta(H-h+1)}-1)\sqrt{\frac{S\iota}{\max\{1,N_{h}^{k}(s,a)\}}},
	\end{align*}
	where the last step holds by Lemma \ref{lem:v_unif_concen} with $c_{0}>0$
	being a universal constant.   Setting $c$ in $b_{h}^{k}$ to
	be equal to $c_{0}$, we have  
	\[
	0\le(q_{h,1}^{k,+}-q_{h,2}^{k})(s,a)\le2b_{h}^{k}.
	\]
	Therefore, we have $q_{h,1}^{k}\ge q_{h,2}^{k}$ by the first inequality
	above, the definition of $q_{h,1}^{k}$ and the property $q_{h,2}^{k}\le e^{\beta(H-h+1)}$.
	Also, since $q_{h,1}^{k,+}\ge q_{h,1}^{k}$, it holds that $q_{h,1}^{k}-q_{h,2}^{k}\le q_{h,1}^{k,+}-q_{h,2}^{k}$.
	The conclusion follows.
	
	\textbf{Case $\beta<0$. }We have, similar to the previous case, that
	\[
	\left|(q_{h,1}^{k,+}-b_{h}^{k}-q_{h,2}^{k})(s,a)\right|\le c_{0}(1-e^{\beta(H-h+1)})\sqrt{\frac{S\iota}{\max\{1,N_{h}^{k}(s,a)\}}}.
	\]
	Choosing $c=c_{0}$ in the definition of $b_{h}^{k}(s,a)$ leads to
	\[
	0\le(q_{h,2}^{k}-q_{h,1}^{k,+})(s,a)\le2b_{h}^{k}.
	\]
	This implies $q_{h,2}^{k}\ge q_{h,1}^{k,+}$, and since $q_{h,1}^{k},q_{h,2}^{k}\ge e^{\beta(H-h+1)}$,
	we also have $q_{h,2}^{k}\ge q_{h,1}^{k}$. In addition, since $q_{h,1}^{k,+}\le q_{h,1}^{k,}$,
	it also holds that $q_{h,2}^{k}-q_{h,1}^{k}\le q_{h,2}^{k}-q_{h,1}^{k,+}$.
	Then the conclusion of this case follows.
\end{proof}

\begin{lem}
	\label{lem:v_G2_bound}On the event of Lemma \ref{lem:v_G1_bound},
	for all $(k,h,s,a)\in[K]\times[H]\times\cS\times\cA$ and any policy
	$\pi$, we have 
	\[
	\begin{cases}
	e^{\beta \cdot Q_{h}^{k}(s,a)}\ge e^{\beta \cdot Q_{h}^{\pi}(s,a)}, & \condrew,\\
	e^{\beta \cdot Q_{h}^{k}(s,a)}\le e^{\beta \cdot Q_{h}^{\pi}(s,a)}, & \condcost.
	\end{cases}
	\]
\end{lem}
\begin{proof}
	We focus on the case of $\beta>0$ since the proof for $\beta<0$
	is very similar. For the purpose of the proof, we set $Q_{H+1}^{\pi}(s,a)=Q_{H+1}^{*}(s,a)=0$
	for all $(s,a)\in\cS\times\cA$. We fix a tuple $(k,s,a)\in[K]\times\cS\times\cA$
	and use strong induction on $h$. The base case for $h=H+1$ is satisfied
	since $e^{\beta\cdot Q_{H+1}^{k}(s,a)}=e^{\beta\cdot Q_{H+1}^{\pi}(s,a)}=1$
	for $k\in[K]$ by definition. Now we fix an $h\in[H]$ and assume
	that $e^{\beta\cdot Q_{h+1}^{k}(s,a)}\ge e^{\beta\cdot Q_{h+1}^{*}(s,a)}$.
	Moreover, by the induction assumption we have 
	\begin{equation}
	e^{\beta\cdot V_{h+1}^{k}(s)}=\max_{a'\in\cA}e^{\beta\cdot Q_{h+1}^{k}(s,a')}\ge\max_{a'\in\cA}e^{\beta\cdot Q_{h+1}^{\pi}(s,a')}\ge e^{\beta\cdot V_{h+1}^{\pi}(s)}.\label{eq:lsvi_V_dominance}
	\end{equation}
	We also assume that $(s,a)$ satisfies $N_{h}^{k}(s,a)\ge1$, since
	otherwise $e^{\beta\cdot Q_{h}^{k}(s,a)}=e^{\beta(H-h+1)}\ge e^{\beta\cdot Q_{h}^{\pi}(s,a)}$
	and we are done. This assumption and Equation \eqref{eq:lsvi_V_dominance}
	together imply $q_{h,2}^{k}\ge q_{h,3}^{k,\pi}$ by Lemma \ref{lem:v_G1_bound}.
	We also have $q_{h,1}^{k}\ge q_{h,2}^{k}$ on the event of Lemma \ref{lem:v_G1_bound}.
	Therefore, it follows that $e^{\beta\cdot Q_{h}^{k}(s,a)}\ge e^{\beta\cdot Q_{h}^{\pi}(s,a)}$
	by Equation \eqref{eq:v_G1_G2_sum_pos} and we have completed the
	induction. 
\end{proof}

\begin{lem}
	\label{lem:v_Vk_dom} For all $(k,h,s)\in[K]\times[H]\times\cS$
	and any $\delta\in(0,1]$, with probability at least $1-\delta/2$,
	we have 
	\[
	\begin{cases}
	e^{\beta \cdot V_{h}^{k}(s)}\ge e^{\beta \cdot V_{h}^{\pi}(s)}, & \condrew,\\
	e^{\beta \cdot V_{h}^{k}(s)}\le e^{\beta \cdot V_{h}^{\pi}(s)}, & \condcost.
	\end{cases}
	\]
\end{lem}
\begin{proof}
	The result follows from Lemma \ref{lem:v_G2_bound} and Equation \eqref{eq:lsvi_V_dominance}. 
\end{proof}

\section{Proof of Theorem \ref{thm:regret_Q_learn} \label{sec:proof_regret_Q_learn}}

We first lay out some additional notations to facilitate our proof.
Let $N_{h}^{k}$, $G_{h}^{k}$, $V_{h}^{k}$ be the $N_{h}$, $G_{h}$,
$V_{h}$ functions at the beginning of the episode $k$, before $t$
is updated. We also set $Q_{h}^{k}\coloneqq\frac{1}{\beta}G_{h}^{k}$.
We let $\widehat{P}_{h}^{k}(\cdot\,|\,s,a)$ denote the delta function
centered at $s_{h+1}^{k}$ for all $(k,h,s,a)\in[K]\times[H]\times\cS\times\cA$.
This means $\E_{s'\sim\widehat{P}_{h}^{k}(\cdot\,|\,s,a)}[f(s')]=f(s_{h+1}^{k})$
for any $f:\cS\to\real$.  Denote by $n_{h}^{k}\coloneqq N_{h}^{k}(s_{h}^{k},a_{h}^{k})$.
Recall from Algorithm \ref{alg:Q_learn}, the learning rate is defined
as 

\begin{equation}
\alpha_{t}\coloneqq\frac{H+1}{H+t},\label{eq:learn_rate}
\end{equation}
for $t\in\mathbb{Z}$.

For now we consider the case for $\beta>0$. We define the following
quantities to ease the notations
for the proof: 
\begin{align*}
\delta_{h}^{k} & \coloneqq e^{\beta\cdot V_{h}^{k}(s_{h}^{k})}-e^{\beta\cdot V_{h}^{\pi^{k}}(s_{h}^{k})},\\
\phi_{h}^{k} & \coloneqq e^{\beta\cdot V_{h}^{k}(s_{h}^{k})}-e^{\beta\cdot V_{h}^{*}(s_{h}^{k})},\\
\xi_{h+1}^{k} & \coloneqq[(P_{h}-\widehat{P}_{h}^{k})(e^{\beta\cdot V_{h+1}^{*}}-e^{\beta\cdot V_{h+1}^{\pi^{k}}})](s_{h}^{k},a_{h}^{k}).
\end{align*}
For each fixed $(k,h)\in[K]\times[H]$, we let $t=N_{h}^{k}(s_{h}^{k},a_{h}^{k})$.
Then it holds that 
\begin{align}
\delta_{h}^{k} & \overset{(i)}{=}e^{\beta\cdot Q_{h}^{k}(s_{h}^{k},a_{h}^{k})}-e^{\beta\cdot Q_{h}^{\pi^{k}}(s_{h}^{k},a_{h}^{k})}\nonumber \\
& =[e^{\beta\cdot Q_{h}^{k}(s_{h}^{k},a_{h}^{k})}-e^{\beta\cdot Q_{h}^{*}(s_{h}^{k},a_{h}^{k})}]+[e^{\beta\cdot Q_{h}^{*}(s_{h}^{k},a_{h}^{k})}-e^{\beta\cdot Q_{h}^{\pi^{k}}(s_{h}^{k},a_{h}^{k})}]\nonumber \\
& \overset{(ii)}{=}[e^{\beta\cdot Q_{h}^{k}(s_{h}^{k},a_{h}^{k})}-e^{\beta\cdot Q_{h}^{*}(s_{h}^{k},a_{h}^{k})}]+e^{\beta\cdot r_{h}(s_{h}^{k},a_{h}^{k})}[P_{h}(e^{\beta\cdot V_{h+1}^{*}}-e^{\beta\cdot V_{h+1}^{\pi^{k}}})](s_{h}^{k},a_{h}^{k})\nonumber \\
& \overset{(iii)}{\le}[e^{\beta\cdot Q_{h}^{k}(s_{h}^{k},a_{h}^{k})}-e^{\beta\cdot Q_{h}^{*}(s_{h}^{k},a_{h}^{k})}]+e^{\beta}[P_{h}(e^{\beta\cdot V_{h+1}^{*}}-e^{\beta\cdot V_{h+1}^{\pi^{k}}})](s_{h}^{k},a_{h}^{k})\nonumber \\
& =[e^{\beta\cdot Q_{h}^{k}(s_{h}^{k},a_{h}^{k})}-e^{\beta\cdot Q_{h}^{*}(s_{h}^{k},a_{h}^{k})}]+e^{\beta}(\delta_{h+1}^{k}-\phi_{h+1}^{k}+\xi_{h+1}^{k})\nonumber \\
& \overset{(iv)}{\le}\alpha_{t}^{0}(e^{\beta(H-h+1)}-1)+2\gamma_{h,t}+\sum_{i\in[t]}\alpha_{t}^{i}\cdot e^{\beta}\left[e^{\beta\cdot V_{h+1}^{k_{i}}(s_{h+1}^{k_{i}})}-e^{\beta\cdot V_{h+1}^{*}(s_{h+1}^{k_{i}})}\right]\nonumber \\
& \quad+e^{\beta}(\delta_{h+1}^{k}-\phi_{h+1}^{k}+\xi_{h+1}^{k})\nonumber \\
& =\alpha_{t}^{0}(e^{\beta(H-h+1)}-1)+2\gamma_{h,t}+\sum_{i\in[t]}\alpha_{t}^{i}\cdot e^{\beta}\phi_{h+1}^{k_{i}}\nonumber \\
& \quad+e^{\beta}(\delta_{h+1}^{k}-\phi_{h+1}^{k}+\xi_{h+1}^{k})\label{eq:qlearn_regret_decomp}
\end{align}
where step $(i)$ holds since $V_{h}^{k}(s_{h}^{k})=\max_{a'\in\cA}Q_{h}^{k}(s_{h}^{k},a')=Q_{h}^{k}(s_{h}^{k},a_{h}^{k})$
and $V_{h}^{\pi^{k}}(s_{h}^{k})=Q_{h}^{\pi^{k}}(s_{h}^{k},\pi_{h}^{k}(s_{h}^{k}))=Q_{h}^{\pi^{k}}(s_{h}^{k},a_{h}^{k})$;
step $(ii)$ holds by the exponential Bellman equation \eqref{eq:exp_bellman};
step $(iii)$ holds since $V_{h+1}^{*}\ge V_{h+1}^{\pi^{k}}$ implies
$e^{\beta\cdot V_{h+1}^{*}}\ge e^{\beta\cdot V_{h+1}^{\pi^{k}}}$
given that $\beta>0$; step $(iv)$ holds on the event of Lemma \ref{lem:upper_bound_Q_diff}
(with $\gamma_{h,t}$ defined therein).

We bound each term in \eqref{eq:qlearn_regret_decomp} one by one.
First, we have 
\begin{align*}
\sum_{k\in[K]}\alpha_{n_{h}^{k}}^{0}(e^{\beta(H-h+1)}-1) & =(e^{\beta(H-h+1)}-1)\sum_{k\in[K]}\indic\{n_{h}^{k}=0\}\\
& \le(e^{\beta(H-h+1)}-1)SA.
\end{align*}
The second term in \eqref{eq:qlearn_regret_decomp} can be bounded
by 
\[
\sum_{k\in[K]}\left(\sum_{i\in[t]}\alpha_{t}^{i}\cdot e^{\beta}\phi_{h+1}^{k_{i}}\right)=\sum_{k\in[K]}\left(\sum_{i\in[n_{h}^{k}]}\alpha_{n_{h}^{k}}^{i}\cdot e^{\beta}\phi_{h+1}^{k_{i}(s_{h}^{k},a_{h}^{k})}\right),
\]
where $k_{i}(s_{h}^{k},a_{h}^{k})$ denotes the episode in which $(s_{h}^{k},a_{h}^{k})$
was taken at step $h$ for the $i$-th time. We re-group the above
summation in a different way. For every $k'\in[K]$, the term $\phi_{h+1}^{k'}$
appears in the summand with $k>k'$ if and only if $(s_{h}^{k},a_{h}^{k})=(s_{h}^{k'},a_{h}^{k'})$.
For the first time we visit $(s_{h}^{k'},a_{h}^{k'})$ we have $n_{h}^{k}=n_{h}^{k'}+1$,
for the second time we have $n_{h}^{k}=n_{h}^{k'}+2$, and etc. Therefore,
we may continue the above display as 
\begin{align*}
\sum_{k\in[K]}\left(\sum_{i\in[n_{h}^{k}]}\alpha_{n_{h}^{k}}^{i}\cdot e^{\beta}\phi_{h+1}^{k_{i}(s_{h}^{k},a_{h}^{k})}\right) & \le\sum_{k'\in[K]}e^{\beta}\phi_{h+1}^{k'}\left(\sum_{t\ge n_{h}^{k'}+1}\alpha_{t}^{n_{h}^{k'}}\right)\\
& \le\left(1+\frac{1}{H}\right)e^{\beta}\sum_{k'\in[K]}\phi_{h+1}^{k'},
\end{align*}
where the last step follows Fact \ref{fact:learn_rate_infinite_sum}.
Collecting the above results and plugging them into Equation \eqref{eq:qlearn_regret_decomp},
we have 
\begin{align}
\sum_{k\in[K]}\delta_{h}^{k} & \le(e^{\beta(H-h+1)}-1)SA+\left(1+\frac{1}{H}\right)e^{\beta}\sum_{k\in[K]}\phi_{h+1}^{k}\nonumber \\
& \quad+\sum_{k\in[K]}e^{\beta}(\delta_{h+1}^{k}-\phi_{h+1}^{k}+\xi_{h+1}^{k})+2\sum_{k\in[K]}\gamma_{h,n_{h}^{k}}\nonumber \\
& \le(e^{\beta(H-h+1)}-1)SA+\left(1+\frac{1}{H}\right)e^{\beta}\sum_{k\in[K]}\delta_{h+1}^{k}\nonumber \\
& \quad+\sum_{k\in[K]}(2\gamma_{h,n_{h}^{k}}+e^{\beta}\xi_{h+1}^{k}),\label{eq:qlearn_regret_interm}
\end{align}
where the last step holds since $\delta_{h+1}^{k}\ge\phi_{h+1}^{k}$
(due to the fact that $\beta>0$ and $V_{h+1}^{*} \ge V_{h+1}^{\pi^{k}}$). 

Now, we unroll the quantity $\sum_{k\in[K]}\delta_{h}^{k}$ recursively
in the form of Equation \eqref{eq:qlearn_regret_interm}, and get
\begin{align}
\sum_{k\in[K]}\delta_{1}^{k} & \le\sum_{h\in[H]}\left[\left(1+\frac{1}{H}\right)e^{\beta}\right]^{h-1}\left[(e^{\beta(H-h+1)}-1)SA+\sum_{k\in[K]}(2\gamma_{h,n_{h}^{k}}+e^{\beta}\xi_{h+1}^{k})\right]\nonumber \\
& =\sum_{h\in[H]}\left(1+\frac{1}{H}\right)^{h-1}\left[(e^{\beta H}-e^{\beta(h-1)})SA+\sum_{k\in[K]}(2e^{\beta(h-1)}\gamma_{h,n_{h}^{k}}+e^{\beta h}\xi_{h+1}^{k})\right]\nonumber \\
& =\sum_{h\in[H]}\left(1+\frac{1}{H}\right)^{h-1}\left[(e^{\beta H}-e^{\beta(h-1)})SA+\sum_{k\in[K]}2e^{\beta(h-1)}\gamma_{h,n_{h}^{k}}\right] \nonumber \\
& \quad +\sum_{h\in[H]}\sum_{k\in[K]}\left(1+\frac{1}{H}\right)^{h-1}e^{\beta h}\xi_{h+1}^{k}\nonumber  \\
& \le e\left[(e^{\beta H}-1)HSA+\sum_{k\in[K]}\sum_{h\in[H]}2e^{\beta(h-1)}\gamma_{h,n_{h}^{k}}\right]+\sum_{h\in[H]}\sum_{k\in[K]}\left(1+\frac{1}{H}\right)^{h-1}e^{\beta h}\xi_{h+1}^{k},\label{eq:qlearn_regret_interm_unrolled}
\end{align}
where the first step uses the fact that $\delta_{H+1}^{k}=0$ for
$k\in[K]$; the last step holds since $(1+1/H)^{h}\le(1+1/H)^{H}\le e$
for all $h\in[H]$. By the pigeonhole principle, for any $h\in[H]$
we have
\begin{align}
\sum_{k\in[K]}\sum_{h\in[H]}e^{\beta(h-1)}\gamma_{n_{h}^{k}} & \lesssim(e^{\beta H}-1)\sum_{k\in[K]}\sqrt{\frac{H\iota}{n_{h}^{k}}}\nonumber \\
& \lesssim(e^{\beta H}-1)\sum_{(s,a)\in\cS\times\cA}\sum_{n\in[N_{h}^{K}(s,a)]}\sqrt{\frac{H\iota}{n}}\nonumber \\
& \lesssim(e^{\beta H}-1)\sqrt{HSAK\iota}\label{eq:qlearn_bonus_bound}
\end{align}
where the third step holds since $\sum_{(s,a)\in\cS\times\cA}N_{h}^{K}(s,a)=K$
and the RHS of the second step is maximized when $N_{h}^{K}(s,a)=K/(SA)$
for all $(s,a)\in\cS\times\cA$. Finally, the Azuma-Hoeffding inequality
and the fact that $\left|\left(1+\frac{1}{H}\right)^{h-1}e^{\beta h}\xi_{h+1}^{k}\right|\le e(e^{\beta H}-1)$
for $h\in[H]$ together imply that with probability at least $1-\delta$,
we have 
\begin{equation}
\left|\sum_{h\in[H]}\sum_{k\in[K]}\left(1+\frac{1}{H}\right)^{h-1}e^{\beta h}\xi_{h+1}^{k}\right|\lesssim(e^{\beta H}-1)\sqrt{HK\iota}.\label{eq:qlearn_mtg_bound}
\end{equation}
Plugging Equations \eqref{eq:qlearn_bonus_bound} and \eqref{eq:qlearn_mtg_bound}
into \eqref{eq:qlearn_regret_interm_unrolled}, we have 
\[
\sum_{k\in[K]}\delta_{1}^{k}\lesssim(e^{\beta H}-1)\sqrt{HSAK\iota},
\]
when $K$ is large enough. Invoking Lemma \ref{lem:reg_decomp} completes
the proof for the case $\beta>0$. 

The proof is very similar for the case of $\beta<0$, and one only
needs to exchange the role of $V_{h}^{k}$ and $V_{h}^{\pi^{k}}$
in the definitions of $\delta_{h}^{k}$, $\phi_{h}^{k}$, $\xi_{h}^{k}$,
etc, to get the counterpart of Equation \eqref{eq:qlearn_regret_decomp}
and of the remaining analysis. 

\subsection{Auxiliary lemmas}

Recall the learning rate $\alpha_{t}$ defined in Equation \eqref{eq:learn_rate}.
We define 
\begin{equation}
\alpha_{t}^{0}\coloneqq\prod_{j=1}^{t}(1-\alpha_{j}),\qquad\alpha_{t}^{i}\coloneqq\alpha_{i}\prod_{j=i+1}^{t}(1-\alpha_{j})\label{eq:learn_rate_prod}
\end{equation}
for integers $i,t\ge1$. We set $\alpha_{t}^{0}=1$ and $\sum_{i\in[t]}\alpha_{t}^{i}=0$
if $t=0$, and $\alpha_{t}^{i}=\alpha_{i}$ if $t<i+1$.

In the following, we provide some useful facts about the learning rate.
\begin{fact}
	\label{fact:learn_rate_prop}The following properties hold for $\alpha_{t}^{i}$.
	\begin{enumerate}[label={(\alph*)},ref={\thefact(\alph*)}]
		\item \label{fact:learn_rate_sum_div_sqrti}$\frac{1}{\sqrt{t}}\le\sum_{i\in[t]}\frac{\alpha_{t}^{i}}{\sqrt{i}}\le\frac{2}{\sqrt{t}}$
		for every integer $t\ge1$.
		\item \label{fact:learn_rate_max_L2} $\max_{i\in[t]}\alpha_{t}^{i}\le\frac{2H}{t}$
		and $\sum_{i\in[t]}(\alpha_{t}^{i})^{2}\le\frac{2H}{t}$ for every
		integer $t\ge1$.
		\item \label{fact:learn_rate_infinite_sum}$\sum_{t=i}^{\infty}\alpha_{t}^{i}=1+\frac{1}{H}$
		for every integer $i\ge1$.
		\item \label{fact:learn_rate_binary}$\sum_{i\in[t]}\alpha_{t}^{i}=1$ and
		$\alpha_{t}^{0}=0$ for every integer $t\ge1$, and $\sum_{i\in[t]}\alpha_{t}^{i}=0$
		and $\alpha_{t}^{0}=1$ for $t=0$.
	\end{enumerate}
\end{fact}
\begin{proof}
	The first three facts can be found in \citet[Lemma 4.1]{jin2018q},
	and the last one follows from direct calculation in view of Equation~\eqref{eq:learn_rate_prod}.
\end{proof}
Define the shorthand $\iota\coloneqq\log(SAT/\delta)$ for $\delta\in(0,1]$.
We fix a tuple $(k,h,s,a)\in[K]\times[H]\times\cS\times\cA$ with
$k_{i}\le k$ being the episode in which $(s,a)$ is visited the $i$-th
time at step $h$. Let us define 
\begin{align*}
q_{h,1}^{k,+}(s,a) & \coloneqq\alpha_{t}^{0}e^{\beta(H-h+1)}+\begin{cases}
\sum_{i\in[t]}\alpha_{t}^{i}\left[e^{\beta[r_{h}(s,a)+V_{h+1}^{k_{i}}(s_{h+1}^{k_{i}})]}+b_{h,i}\right], & \condrew,\\
\sum_{i\in[t]}\alpha_{t}^{i}\left[e^{\beta[r_{h}(s,a)+V_{h+1}^{k_{i}}(s_{h+1}^{k_{i}})]}-b_{h,i}\right], & \condcost,
\end{cases}\\
q_{h,1}^{k}(s,a) & \coloneqq\begin{cases}
\min\{q_{h,1}^{k,+}(s,a),e^{\beta(H-h+1)}\}, & \condrew,\\
\max\{q_{h,1}^{k,+}(s,a),e^{\beta(H-h+1)}\}, & \condcost,
\end{cases}
\end{align*}
and 
\begin{align*}
q_{h,2}^{k,\circ}(s,a) & \coloneqq\alpha_{t}^{0}e^{\beta(H-h+1)}+\sum_{i\in[t]}\alpha_{t}^{i}\left[e^{\beta[r_{h}(s,a)+V_{h+1}^{*}(s_{h+1}^{k_{i}})]}\right]\\
q_{h,2}^{k,+}(s,a) & \coloneqq\alpha_{t}^{0}e^{\beta(H-h+1)}+\begin{cases}
\sum_{i\in[t]}\alpha_{t}^{i}\left[e^{\beta[r_{h}(s,a)+V_{h+1}^{*}(s_{h+1}^{k_{i}})]}+b_{h,i}\right], & \condrew,\\
\sum_{i\in[t]}\alpha_{t}^{i}\left[e^{\beta[r_{h}(s,a)+V_{h+1}^{*}(s_{h+1}^{k_{i}})]}-b_{h,i}\right], & \condcost,
\end{cases}\\
q_{h,2}^{k}(s,a) & \coloneqq\begin{cases}
\min\{q_{h,2}^{k,+}(s,a),e^{\beta(H-h+1)}\}, & \condrew,\\
\max\{q_{h,2}^{k,+}(s,a),e^{\beta(H-h+1)}\}, & \condcost,
\end{cases}
\end{align*}
and 
\[
q_{h,3}^{k}(s,a)\coloneqq\alpha_{t}^{0}e^{\beta\cdot Q_{h}^{*}(s,a)}+\sum_{i\in[t]}\alpha_{t}^{i}\left[\E_{s'\sim P_{h}(\cdot\,|\,s,a)}e^{\beta[r_{h}(s,a)+V_{h+1}^{*}(s')]}\right].
\]

We have a simple fact on $q_{h,2}^{k}$ and $q_{h,2}^{k,\circ}$. 
\begin{fact}
	\label{fact:q_2_prime}If $\beta>0$, we have $q_{h,2}^{k,\circ}(\cdot,\cdot)\le q_{h,2}^{k}(\cdot,\cdot)$;
	if $\beta<0$, we have $q_{h,2}^{k,\circ}(\cdot,\cdot)\ge q_{h,2}^{k}(\cdot,\cdot)$.
\end{fact}
\begin{proof}
	We focus on the case where $\beta>0$ and the case for $\beta<0$
	can be proved similarly. Note that $r_{h}(s,a)+V_{h+1}^{*}(s_{h+1}^{k_{i}})\in[0,H-h+1]$
	implies $e^{\beta[r_{h}(s,a)+V_{h+1}^{*}(s_{h+1}^{k_{i}})]}\le e^{\beta(H-h+1)}$.
	We also have $\alpha_{t}^{0},\sum_{i\in[t]}\alpha_{t}^{i}\in\{0,1\}$
	with $\alpha_{t}^{0}+\sum_{i\in[t]}\alpha_{t}^{i}=1$ by Fact \ref{fact:learn_rate_binary}.
	Together they imply that $q_{h,2}^{k,\circ}(\cdot,\cdot)\le e^{\beta(H-h+1)}$
	and $(q_{h,2}^{k,\circ}-q_{h,2}^{k,+})(\cdot,\cdot)=-\sum_{i\in[t]}\alpha_{t}^{i}b_{h,i}\le0$
	by definition of $b_{h,i}$ in Line \ref{line:qlearn_bonus_def} of
	Algorithm \ref{alg:Q_learn}. Therefore, $q_{h,2}^{k,\circ}(\cdot,\cdot)\le\min\{e^{\beta(H-h+1)},q_{h,2}^{k,+}(\cdot,\cdot)\}=q_{h,2}^{k}(\cdot,\cdot)$.
\end{proof}
Next, we write the difference $e^{\beta\cdot Q_{h}^{k}}-e^{\beta\cdot Q_{h}^{*}}$
in terms of $q_{h,1}^{k}$ and $q_{h,3}^{k}$.
\begin{lem}
	\label{lem:Q_identity} \emph{For any $(k,h,s,a)\in[K]\times[H]\times\cS\times\cA$,
		suppose $(s,a)$ was previously visited at step $h$ of episodes $k_{1},\ldots,k_{t}<k$.
		We have 
		\[
		(e^{\beta\cdot Q_{h}^{k}}-e^{\beta\cdot Q_{h}^{*}})(s,a)=(q_{h,1}^{k}-q_{h,3}^{k})(s,a).
		\]
	}
\end{lem}
\begin{proof}
	For $e^{\beta\cdot Q_{h}^{k}}$, Line \ref{line:qlearn_Q_update}
	of Algorithm \ref{alg:Q_learn} implies that 
	\begin{equation}
	e^{\beta\cdot Q_{h}^{k}(s,a)}=q_{h,1}^{k}(s,a).\label{eq:Qk_rewrite}
	\end{equation}
	For $e^{\beta\cdot Q_{h}^{*}}$, we have from exponential Bellman
	equation \eqref{eq:exp_bellman} that 
	\[
	e^{\beta\cdot Q_{h}^{*}(s,a)}=e^{\beta\cdot r_{h}(s,a)}\left[\E_{s'\sim P_{h}(\cdot\,|\,s,a)}e^{\beta\cdot V_{h+1}^{*}(s')}\right].
	\]
	Let $t=N_{h}^{k}(s,a)$ and by Fact \ref{fact:learn_rate_binary},
	we have 
	\begin{align*}
	e^{\beta\cdot Q_{h}^{*}(s,a)} & =\alpha_{t}^{0}e^{\beta\cdot Q_{h}^{*}(s,a)}+\sum_{i\in[t]}\alpha_{t}^{i}e^{\beta\cdot r_{h}(s,a)}\left[\E_{s'\sim P_{h}(\cdot\,|\,s,a)}e^{\beta\cdot V_{h+1}^{*}(s')}\right]
	\end{align*}
	for each integer $t\ge0$. By the definition of $q_{h,3}^{k}$ we
	have 
	\begin{align}
	e^{\beta\cdot Q_{h}^{*}(s,a)} & =q_{h,3}^{k}(s,a).\label{eq:Qstar_rewrite}
	\end{align}
	The proof is completed by combining Equations \eqref{eq:Qk_rewrite}
	and \eqref{eq:Qstar_rewrite}.
\end{proof}
From Lemma \ref{lem:Q_identity}, we can derive the decomposition
\begin{equation}
(e^{\beta\cdot Q_{h}^{k}}-e^{\beta\cdot Q_{h}^{*}})(s,a)=(q_{h,1}^{k}-q_{h,2}^{k})(s,a)+(q_{h,2}^{k}-q_{h,3}^{k})(s,a)\label{eq:qlearn_q_diff_rew}
\end{equation}
if $\beta>0$, and 
\begin{equation}
(e^{\beta\cdot Q_{h}^{k}}-e^{\beta\cdot Q_{h}^{*}})(s,a)=(q_{h,2}^{k}-q_{h,1}^{k})(s,a)+(q_{h,3}^{k}-q_{h,2}^{k})(s,a)\label{eq:qlearn_q_diff_cost}
\end{equation}
if $\beta<0$. We have the following lemmas.
\begin{lem}
	\label{lem:weighted_sum_mtg_bonus_bound} There exists a universal
	constant $c>0$ in the definition of $b_{h,t}$ in Algorithm \ref{alg:Q_learn}
	such that for any $(k,h,s,a)\in[K]\times[H]\times\cS\times\cA$ and
	$k_{1},\ldots,k_{t}<k$ with $t=N_{h}^{k}(s,a)$, we have 
	\begin{align*}
	& \quad\left|\sum_{i\in[t]}\alpha_{t}^{i}\left[e^{\beta[r_{h}(s,a)+V_{h+1}^{*}(s_{h+1}^{k_{i}})]}-\E_{s'\sim P_{h}(\cdot\,|\,s,a)}[e^{\beta[r_{h}(s,a)+V_{h+1}^{*}(s')]}]\right]\right|\\
	& \le c\left|e^{\beta(H-h+1)}-1\right|\sqrt{\frac{H\iota}{t}}.
	\end{align*}
	with probability at least $1-\delta$, and 
	\[
	\sum_{i\in[t]}\alpha_{t}^{i}b_{h,i}\in\left[c\left|e^{\beta(H-h+1)}-1\right|\sqrt{\frac{H\iota}{t}},2c\left|e^{\beta(H-h+1)}-1\right|\sqrt{\frac{H\iota}{t}}\right].
	\]
\end{lem}
\begin{proof}
	We focus on the case where $\beta>0$ and the proof for $\beta<0$
	is similar. For any $(k,h,s,a)\in[K]\times[H]\times\cS\times\cA$,
	define 
	\begin{align*}
	\psi(i,k,h,s,a) & \coloneqq e^{\beta[r_{h}(s,a)+V_{h+1}^{*}(s_{h+1}^{k_{i}})]}-\E_{s'\sim P_{h}(\cdot\,|\,s,a)}[e^{\beta[r_{h}(s,a)+V_{h+1}^{*}(s')]}]\\
	& =\E_{s'\sim\hat{P}_{h}^{k_{i}}(\cdot\,|\,s,a)}[e^{\beta[r_{h}(s,a)+V_{h+1}^{*}(s')]}]-\E_{s'\sim P_{h}(\cdot\,|\,s,a)}[e^{\beta[r_{h}(s,a)+V_{h+1}^{*}(s')]}]
	\end{align*}
	Let us fix a tuple $(k,h,s,a)\in[K]\times[H]\times\cS\times\cA$.
	We have that $\{\indic(k_{i}\le K)\cdot\psi(i,k,h,s,a)\}_{i\in[\tau]}$
	for $\tau\in[K]$ is a martingale difference sequence. By the Azuma-Hoeffding
	inequality and a union bound over $\tau\in[K]$, it holds that with
	probability at least $1-\delta/(HSA)$, for all $\tau\in[K]$, 
	\begin{align*}
	& \left|\sum_{i\in[\tau]}\alpha_{\tau}^{i}\cdot\indic(k_{i}\le K)\cdot\psi(i,k,h,s,a)\right|\\
	& \le\frac{c}{2}(e^{\beta(H-h+1)}-1)\sqrt{\iota\sum_{i\in[\tau]}(\alpha_{\tau}^{i})^{2}}\le c(e^{\beta(H-h+1)}-1)\sqrt{\frac{H\iota}{\tau}}
	\end{align*}
	where $c>0$ is some universal constant, the first step holds since
	$r_{h}(s,a)+V_{h+1}^{*}(s')\in[0,H-h+1]$ for $s'\in\cS$, and the
	last step follows from Fact \ref{fact:learn_rate_max_L2}. Since the
	above equation holds for all $\tau\in[K]$, it also holds for $\tau=t=N_{h}^{k}(s,a)\le K$.
	Note that $\indic(k_{i}\le K)=1$ for all $i\in[N_{h}^{k}(s,a)]$.
	Therefore, applying another union bound over $(h,s,a)\in[H]\times\cS\times\cA$,
	we have that the following holds for all $(k,h,s,a)\in[K]\times[H]\times\cS\times\cA$
	and with probability at least $1-\delta$:
	\begin{align}
	\left|\sum_{i\in[t]}\alpha_{\tau}^{i}\cdot\psi(i,k,h,s,a)\right| & \le c(e^{\beta(H-h+1)}-1)\sqrt{\frac{H\iota}{t}},\label{eq:mds_weighted_bound}
	\end{align}
	where $t=N_{h}^{k}(s,a)$. Using the fact that $r_{h}+V_{h+1}^{*}\in[0,H-h+1]$,
	we have 
	\begin{align*}
	& \quad\left|\sum_{i\in[t]}\alpha_{t}^{i}\left[\E_{s'\sim\hat{P}_{h}^{k_{i}}(\cdot\,|\,s,a)}e^{\beta[r_{h}(s,a)+V_{h+1}^{*}(s')]}-\E_{s'\sim P_{h}(\cdot\,|\,s,a)}e^{\beta[r_{h}(s,a)+V_{h+1}^{*}(s')]}\right]\right|\\
	& =\left|\sum_{i\in[t]}\alpha_{t}^{i}\cdot\psi(i,k,h,s,a)\right|\le c(e^{\beta(H-h+1)}-1)\sqrt{\frac{H\iota}{t}}.
	\end{align*}
	
	For bounds on $\sum_{i\in[t]}\alpha_{t}^{i}b_{h,i}$, we recall the
	definition of $\{b_{h,t}\}$ in Line \ref{line:qlearn_bonus_def}
	of Algorithm \ref{alg:Q_learn} and compute 
	\begin{align*}
	\sum_{i\in[t]}\alpha_{t}^{i}b_{h,i} & =c(e^{\beta(H-h+1)}-1)\sum_{i\in[t]}\alpha_{t}^{i}\sqrt{\frac{H\iota}{i}}\\
	& \in\left[c(e^{\beta(H-h+1)}-1)\sqrt{\frac{H\iota}{t}},2c(e^{\beta(H-h+1)}-1)\sqrt{\frac{H\iota}{t}}\right]
	\end{align*}
	where the last step holds by Fact \ref{fact:learn_rate_sum_div_sqrti}.
\end{proof}
The next two lemmas compare the iterate $e^{\beta\cdot Q_{h}^{k}}$
(and $e^{\beta\cdot V_{h}^{k}}$) with the optimal exponential value
function $e^{\beta\cdot Q_{h}^{*}}$ (and $e^{\beta\cdot V_{h}^{*}}$). 
\begin{lem}
	\label{lem:lower_bound_Q_diff}For all $(k,h,s,a)$and any $\delta\in(0,1]$,
	it holds with probability at least $1-\delta$ that 
	\[
	\begin{cases}
	e^{\beta \cdot Q_{h}^{k}(s,a)}\ge e^{\beta \cdot Q_{h}^{*}(s,a)}, & \condrew,\\
	e^{\beta \cdot Q_{h}^{k}(s,a)}\le e^{\beta \cdot Q_{h}^{*}(s,a)}, & \condcost.
	\end{cases}
	\]
\end{lem}
\begin{proof}
	We focus on the case where $\beta>0$ and the proof for $\beta<0$
	is similar. For the purpose of the proof, we set $Q_{H+1}^{k}(s,a)=Q_{H+1}^{*}(s,a)=0$
	for all $(k,s,a)\in[K]\times\cS\times\cA$. We fix a $(s,a)\in\cS\times\cA$
	and use strong induction on $k$ and $h$. Without loss of generality,
	we assume that there exists a $(k,h)$ such that $(s,a)=(s_{h}^{k},a_{h}^{k})$
	(that is, $(s,a)$ has been visited at some point in Algorithm \ref{alg:Q_learn}),
	since otherwise $e^{\beta\cdot Q_{h}^{k}(s,a)}=e^{\beta(H-h+1)}\ge e^{\beta\cdot Q_{h}^{*}(s,a)}$
	for all $(k,h)\in[K]\times[H]$ and we are done. 
	
	The base case for $k=1$ and $h=H+1$ is satisfied since $e^{\beta\cdot Q_{H+1}^{k'}(s,a)}=e^{\beta\cdot Q_{H+1}^{*}(s,a)}$
	for $k'\in[K]$ by definition. We fix a $(k,h)\in[K]\times[H]$ and
	assume that $e^{\beta\cdot Q_{h+1}^{k_{i}}(s,a)}\ge e^{\beta\cdot Q_{h+1}^{*}(s,a)}$
	for each $k_{1},\ldots,k_{t}<k$ (here $t=N_{h}^{k}(s,a)$). Then
	we have for $i\in[t]$ that 
	\[
	e^{\beta\cdot V_{h+1}^{k_{i}}(s)}=\max_{a'\in\cA}e^{\beta\cdot Q_{h+1}^{k_{i}}(s,a')}\ge\max_{a'\in\cA}e^{\beta\cdot Q_{h+1}^{*}(s,a')}=e^{\beta\cdot V_{h+1}^{*}(s)},
	\]
	where the first equality holds by the update procedure in Algorithm
	\ref{alg:Q_learn}. Recall the decomposition in Equation \eqref{eq:qlearn_q_diff_rew}.
	The above displayed equation implies $q_{h,1}^{k}\ge q_{h,2}^{k}$
	by the definition of $q_{h,2}^{k}$. We also have $q_{h,2}^{k}\ge q_{h,3}^{k}$
	by the fact $e^{\beta\cdot Q_{h}^{*}(s,a)}\le e^{\beta(H-h+1)}$ and
	on the event of Lemma \ref{lem:weighted_sum_mtg_bonus_bound}. Therefore,
	it follows that $(e^{\beta\cdot Q_{h}^{k}}-e^{\beta\cdot Q_{h}^{*}})(s,a)\ge0$
	by Equation \eqref{eq:qlearn_q_diff_rew}. The induction is completed.
\end{proof}
\begin{lem}
	\label{lem:upper_bound_Q_diff}For all $(k,h,s,a)\in[K]\times[H]\times\cS\times\cA$
	such that $t=N_{h}^{k}(s,a)\ge1$, let $\gamma_{h,t}\coloneqq2\sum_{i\in[t]}\alpha_{t}^{i}b_{h,i}$
	and let $k_{1},\ldots,k_{t}<k$ be the episodes in which $(s,a)$
	is visited at step $h$. Then the following holds with probability
	at least $1-\delta$: if $\beta>0$, we have 
	\begin{align*}
	& \quad(e^{\beta\cdot Q_{h}^{k}}-e^{\beta\cdot Q_{h}^{*}})(s,a)\\
	& \le\alpha_{t}^{0}\left[e^{\beta(H-h+1)}-1\right]+2\gamma_{h,t}+\sum_{i\in[t]}\alpha_{t}^{i}e^{\beta}\left[e^{\beta\cdot V_{h+1}^{k_{i}}(s_{h+1}^{k_{i}})}-e^{\beta\cdot V_{h+1}^{*}(s_{h+1}^{k_{i}})}\right],
	\end{align*}
	and if $\beta<0$, we have 
	\begin{align*}
	& \quad(e^{\beta\cdot Q_{h}^{*}}-e^{\beta\cdot Q_{h}^{k}})(s,a)\\
	& \le\alpha_{t}^{0}\left[1-e^{\beta(H-h+1)}\right]+2\gamma_{h,t}+\sum_{i\in[t]}\alpha_{t}^{i}\left[e^{\beta\cdot V_{h+1}^{*}(s_{h+1}^{k_{i}})}-e^{\beta\cdot V_{h+1}^{k_{i}}(s_{h+1}^{k_{i}})}\right].
	\end{align*}
	Furthermore, we have $\gamma_{h,t}\le4c\left|e^{\beta(H-h+1)}-1\right|\sqrt{\frac{H\iota}{t}}$.
\end{lem}
\begin{proof}
	Note that by definition, 
	\[
	q_{h,1}^{k}(s,a)=e^{\beta\cdot Q_{h}^{k}(s,a)},\quad q_{h,3}^{k}(s,a)=e^{\beta\cdot Q_{h}^{*}(s,a)}.
	\]
	\textbf{}Let us fix a tuple $(k,h,s,a)\in[K]\times[H]\times\cS\times\cA$.
	On the event of Lemma \ref{lem:lower_bound_Q_diff}, we have 
	\[
	\begin{cases}
	e^{\beta\cdot Q_{h}^{k}(s,a)}\ge e^{\beta\cdot Q_{h}^{*}(s,a)}, & \condrew,\\
	e^{\beta\cdot Q_{h}^{k}(s,a)}\le e^{\beta\cdot Q_{h}^{*}(s,a)}, & \condcost.
	\end{cases}
	\]
	This implies that for $i\in[t]$, if $\beta>0$ then
	\[
	e^{\beta\cdot V_{h+1}^{k_{i}}(s)}=\max_{a'\in\cA}e^{\beta\cdot Q_{h+1}^{k_{i}}(s,a')}\ge\max_{a'\in\cA}e^{\beta\cdot Q_{h+1}^{*}(s,a')}=e^{\beta\cdot V_{h+1}^{*}(s)},
	\]
	and if $\beta<0$, then 
	\[
	e^{\beta\cdot V_{h+1}^{k_{i}}(s)}=\min_{a'\in\cA}e^{\beta\cdot Q_{h+1}^{k_{i}}(s,a')}\le\min_{a'\in\cA}e^{\beta\cdot Q_{h+1}^{*}(s,a')}=e^{\beta\cdot V_{h+1}^{*}(s)}.
	\]
	Here, the first equalities for the above two displays follow from
	the update procedure in Algorithm \ref{alg:Q_learn}. 
	
	\textbf{Case $\beta>0$.} We have 
	\begin{align*}
	(q_{h,1}^{k}-q_{h,2}^{k})(s,a) & \overset{(i)}{\le}(q_{h,1}^{k,+}-q_{h,2}^{k,\circ})(s,a)\\
	& \overset{(ii)}{\le}\sum_{i\in[t]}\alpha_{t}^{i}\left[e^{\beta[r_{h}(s,a)+V_{h+1}^{k_{i}}(s_{h+1}^{k_{i}})]}-e^{\beta[r_{h}(s,a)+V_{h+1}^{*}(s_{h+1}^{k_{i}})]}\right]+\sum_{i\in[t]}\alpha_{t}^{i}b_{h,i}\\
	& \le\sum_{i\in[t]}\alpha_{t}^{i}\cdot e^{\beta}\left[e^{\beta\cdot V_{h+1}^{k_{i}}(s_{h+1}^{k_{i}})}-e^{\beta\cdot V_{h+1}^{*}(s_{h+1}^{k_{i}})}\right]+\gamma_{h,t}
	\end{align*}
	where step $(i)$ holds by the fact that $\alpha_{t}^{0},\sum_{i\in[t]}\alpha_{t}^{i}\in\{0,1\}$
	with $\alpha_{t}^{0}+\sum_{i\in[t]}\alpha_{t}^{i}=1$ by Fact \ref{fact:learn_rate_binary}
	(so that $q_{h,1}^{k}\ge q_{h,2}^{k}$); step $(ii)$ holds by definitions of $q^{k,+}_{h,1}$ and $q^{k,\circ}_{h,2}$; the last step holds since $r_{h}$ is in $[0,1]$ entrywise and
	$V_{h+1}^{k_{i}}(s)\ge V_{h+1}^{*}(s)$. Moreover, we have 
	\begin{align*}
	(q_{h,2}^{k}-q_{h,3}^{k})(s,a) & \overset{(i)}{\le}(q_{h,2}^{k,+}-q_{h,3}^{k})(s,a)\\
	& =\alpha_{t}^{0}\left[e^{\beta(H-h+1)}-e^{\beta\cdot Q_{h}^{*}(s,a)}\right]+\sum_{i\in[t]}\alpha_{t}^{i}b_{h,i}\\
	& \quad+\sum_{i\in[t]}\alpha_{t}^{i}\left[e^{\beta[r_{h}(s,a)+V_{h+1}^{*}(s_{h+1}^{k_{i}})]}-\E_{s'\sim P_{h}(\cdot\,|\,s,a)}[e^{\beta[r_{h}(s,a)+V_{h+1}^{*}(s')]}]\right]\\
	& \le\alpha_{t}^{0}\left[e^{\beta(H-h+1)}-1\right]+\gamma_{h,t},
	\end{align*}
	where step $(i)$ holds by 
	\[
	\sum_{i\in[t]}\alpha_{t}^{i}b_{h,i}\ge\left|\sum_{i\in[t]}\alpha_{t}^{i}\left[e^{\beta[r_{h}(s,a)+V_{h+1}^{*}(s_{h+1}^{k_{i}})]}-\E_{s'\sim P_{h}(\cdot\,|\,s,a)}[e^{\beta[r_{h}(s,a)+V_{h+1}^{*}(s')]}]\right]\right|
	\]
	on the event of Lemma \ref{lem:weighted_sum_mtg_bonus_bound} (so
	that $q_{h,2}^{k}\ge q_{h,3}^{k}$) and Fact \ref{fact:q_2_prime};
	the last step holds by $Q_{h}^{*}\ge0$ and on the event of Lemma
	\ref{lem:weighted_sum_mtg_bonus_bound}.
	
	\textbf{Case $\beta<0$.}\emph{ }We have 
	\begin{align*}
	(q_{h,2}^{k}-q_{h,1}^{k})(s,a) & \overset{(i)}{\le}(q_{h,2}^{k,\circ}-q_{h,1}^{k,+})(s,a)\\
	& =\sum_{i\in[t]}\alpha_{t}^{i}\left[e^{\beta[r_{h}(s,a)+V_{h+1}^{*}(s_{h+1}^{k_{i}})]}-e^{\beta[r_{h}(s,a)+V_{h+1}^{k_{i}}(s_{h+1}^{k_{i}})]}\right]+\sum_{i\in[t]}\alpha_{t}^{i}b_{i}\\
	& \le\sum_{i\in[t]}\alpha_{t}^{i}\left[e^{\beta\cdot V_{h+1}^{*}(s_{h+1}^{k_{i}})}-e^{\beta\cdot V_{h+1}^{k_{i}}(s_{h+1}^{k_{i}})}\right]+\gamma_{h,t},
	\end{align*}
	where the step $(i)$ holds since $q_{h,2}^{k,\circ}\ge q_{h,2}^{k}$
	by Fact \ref{fact:q_2_prime} and $q_{h,1}^{k,+}\le q_{h,1}^{k}$
	by definition, and the last step holds by the fact that $r_{h}(s,a)+V_{h+1}^{k_{i}}(s)\ge r_{h}(s,a)+V_{h+1}^{*}(s)$,
	that $e^{\beta\cdot r_{h}(s,a)}\le1$ given $\beta<0$, and the definition
	of $\gamma_{h,t}$. In addition, we can derive 
	\begin{align*}
	(q_{h,3}^{k}-q_{h,2}^{k})(s,a) & \overset{(i)}{\le}(q_{h,3}^{k}-q_{h,2}^{k,+})(s,a)\\
	& =\alpha_{t}^{0}\left[1-e^{\beta\cdot Q_{h}^{*}(s,a)}\right]+\sum_{i\in[t]}\alpha_{t}^{i}b_{i}\\
	& \quad+\sum_{i\in[t]}\alpha_{t}^{i}\left[\E_{s'\sim P_{h}(\cdot\,|\,s,a)}[e^{\beta[r_{h}(s,a)+V_{h+1}^{*}(s')]}]-e^{\beta[r_{h}(s,a)+V_{h+1}^{*}(s_{h+1}^{k_{i}})]}\right]\\
	& \overset{(ii)}{\le}\alpha_{t}^{0}\left[1-e^{\beta(H-h+1)}\right]+2\sum_{i\in[t]}\alpha_{t}^{i}b_{i}\\
	& \le\alpha_{t}^{0}\left[1-e^{\beta(H-h+1)}\right]+\gamma_{h,t}.
	\end{align*}
	where step $(i)$ holds since $q_{h,2}^{k}\ge q_{h,2}^{k,+}$, step
	$(ii)$ holds on the event of Lemma \ref{lem:weighted_sum_mtg_bonus_bound},
	and the last step holds by the definition of $\gamma_{h,t}$.
	
	Combining the above calculations with Equation \eqref{eq:qlearn_q_diff_rew}
	for the case where $\beta>0$ (or Equation \eqref{eq:qlearn_q_diff_cost}
	for the case where $\beta<0$) yields the upper bound for $(e^{\beta\cdot Q_{h}^{k}}-e^{\beta\cdot Q_{h}^{*}})(s,a)$
	(or $(e^{\beta\cdot Q_{h}^{*}}-e^{\beta\cdot Q_{h}^{k}})(s,a)$).
	Furthermore, Lemma \ref{lem:weighted_sum_mtg_bonus_bound} and the
	definition of $\gamma_{h,t}$ together imply
	\[
	\gamma_{h,t}\le4c\left|e^{\beta(H-h+1)}-1\right|\sqrt{\frac{H\iota}{t}}.
	\]
	The proof is completed.
\end{proof}


\vspace{1em}
We present a simple inequality for the regret.
\begin{lem}
	\label{lem:reg_decomp} Suppose that for any $k\in[K]$ we have $V_{1}^{k}(s_1^k)\ge V_{1}^{*}(s_1^k)$. Then for $\beta > 0$, the regret
	is bounded by 
	\[
	\reg(K)\le\frac{1}{\beta}\sum_{k\in[K]} [e^{\beta\cdot V_{1}^{k}(s_{1}^{k})}-e^{\beta\cdot V_{1}^{\pi^{k}}(s_{1}^{k})}],
	\]
	and for $\beta < 0$, the regret is bounded by
	\[
	\reg(K)\le\frac{e^{-\beta H}}{\left|\beta\right|}\sum_{k\in[K]}[e^{\beta\cdot V_{1}^{\pi^{k}}(s_{1}^{k})}-e^{\beta\cdot V_{1}^{k}(s_{1}^{k})}],
	\]
\end{lem}
\begin{proof}
	For $\beta>0$, we have 
	\begin{align*}
	\reg(K) & =\sum_{k\in[K]}(V_{1}^{*}-V_{1}^{\pi^{k}})(s_{1}^{k})\\
	& \overset{(i)}{\le}\sum_{k\in[K]}(V_{1}^{k}-V_{1}^{\pi^{k}})(s_{1}^{k})\\
	& =\sum_{k\in[K]}\left[\frac{1}{\beta}\log\{e^{\beta\cdot V_{1}^{k}(s_{1}^{k})}\}-\frac{1}{\beta}\log\{e^{\beta\cdot V_{1}^{\pi^{k}}(s_{1}^{k})}\}\right]\\
	& \overset{(ii)}{\le}\sum_{k\in[K]}\frac{1}{\beta}[e^{\beta\cdot V_{1}^{k}(s_{1}^{k})}-e^{\beta\cdot V_{1}^{\pi^{k}}(s_{1}^{k})}]\\
	& =\frac{1}{\beta}\sum_{k\in[K]} [e^{\beta\cdot V_{1}^{k}(s_{1}^{k})}-e^{\beta\cdot V_{1}^{\pi^{k}}(s_{1}^{k})}],
	\end{align*}
	where step $(i)$ holds by our assumption,
	and
	step $(ii)$ holds by the 1-Lipschitzness of the function $f(x)=\log x$
	for $x\ge1$ and note that our assumption implies that $V_{1}^{k}(s_1^k)\ge V_{1}^{*}(s_1^k)\ge V_{1}^{\pi^{k}}(s_1^k)$.

For $\beta<0$, we similarly have 
\begin{align*}
\reg(K) & =\sum_{k\in[K]}(V_{1}^{*}-V_{1}^{\pi^{k}})(s_{1}^{k})\\
& \overset{(i)}{\le}\sum_{k\in[K]}(V_{1}^{k}-V_{1}^{\pi^{k}})(s_{1}^{k})\\
& =\sum_{k\in[K]}\left[\frac{1}{\beta}\log\{e^{\beta\cdot V_{1}^{k}(s_{1}^{k})}\}-\frac{1}{\beta}\log\{e^{\beta\cdot V_{1}^{\pi^{k}}(s_{1}^{k})}\}\right]\\
& =\sum_{k\in[K]}\left[\frac{1}{(-\beta)}\log\{e^{\beta\cdot V_{1}^{\pi^{k}}(s_{1}^{k})}\}-\frac{1}{(-\beta)}\log\{e^{\beta\cdot V_{1}^{k}(s_{1}^{k})}\}\right]\\
& \overset{(ii)}{\le}\sum_{k\in[K]}\frac{e^{-\beta H}}{(-\beta)}[e^{\beta\cdot V_{1}^{\pi^{k}}(s_{1}^{k})}-e^{\beta\cdot V_{1}^{k}(s_{1}^{k})}]\\
& =\frac{e^{-\beta H}}{\left|\beta\right|}\sum_{k\in[K]}[e^{\beta\cdot V_{1}^{\pi^{k}}(s_{1}^{k})}-e^{\beta\cdot V_{1}^{k}(s_{1}^{k})}],
\end{align*}
where step $(i)$ holds by our assumption,
and
step $(ii)$ holds by the $(e^{-\beta H})$-Lipschitzness of the function
$f(x)=\log x$ for $x\ge e^{\beta H}$ and note that our assumption implies
that $V_{1}^{k}(s_1^k)\ge V_{1}^{*}(s_1^k)\ge V_{1}^{\pi^{k}}(s_1^k)$.
\end{proof}

\vspace{2em}

\paragraph{Broader impact and future directions.}

Risk-sensitive RL has close association with neuroscience, psychology and behavioral economics, as it has been applied to model human behaviors \citep{niv2012neural,shen2014risk}. Interestingly, this array of topics are also actively studied by researchers in the areas of meta learning \citep{xu2021meta},  biologically inspired deep learning \citep{song2021convergence} and deep reinforcement learning \citep{leibo2018psychlab}. It would be an exciting research direction to establish connections between these related areas through rigorous and theoretical analysis of deep learning  \citep{chen2020deep,chen2020multiple}. Motivated by the inertia of switching actions that is widely observed in human behaviors, the study of switching constrained algorithms \citep{chen2019minimax} for risk-sensitive RL could be another promising direction for future investigation.  Furthermore, to make our algorithms practical and efficient on  large-scaled datasets collected in the aforementioned applications, it is imperative to enable offline learning procedures for risk-sensitive RL, possibly by techniques developed in the literature of offline RL \citep{chen2021infinite}. It would also be of great interest to understand the  landscape of the optimization problems \citep{ling2019landscape} that arise in the offline learning setting. 

\end{document}